\documentclass{article}

 \usepackage[main, final]{neurips_2026}

\usepackage{bm}
\usepackage{bbm}
\usepackage{multirow}
\usepackage{pifont}
\usepackage{algorithm}
\usepackage{algorithmic}
\usepackage{wrapfig}
\usepackage{colortbl}
\usepackage{booktabs}

\usepackage{amsmath}
\usepackage{amssymb}
\usepackage{mathtools}
\usepackage{amsthm}
\usepackage[most]{tcolorbox}
\usepackage{graphicx}
\usepackage{hyperref,xcolor}
\usepackage{enumitem}
\usepackage{fontawesome5}
\theoremstyle{plain}
\newtheorem{theorem}{Theorem}[section]
\newtheorem{proposition}[theorem]{Proposition}
\newtheorem{lemma}[theorem]{Lemma}

\theoremstyle{definition}

\newtheorem{assumption}[theorem]{Assumption}
\theoremstyle{remark}

\definecolor{midnightgreen}{rgb}{0.0, 0.29, 0.33}
\definecolor{deepgreen}{HTML}{055c29}
\definecolor{deeppurple}{HTML}{7030a0}
\definecolor{deepblue}{HTML}{171d91}
\definecolor{brown}{HTML}{843c0c}
\definecolor{shadered}{HTML}{ffe5e5}
\definecolor{shadegreen}{HTML}{e5f7ed}
\definecolor{msftBlack}{RGB}{0,0,0}
\definecolor{lightred}{RGB}{255,163,163}
\definecolor{deepred}{RGB}{153,0,0}
\definecolor{softblue}{RGB}{30, 90, 160}

\definecolor{barblue}{RGB}{90,120,180}
\definecolor{barorange}{RGB}{225,124,5}

\definecolor{DeltaBg}{HTML}{D4F2D7} 
\definecolor{SearchBg}{HTML}{C2E6F5} 
\definecolor{AgenticBg}{HTML}{F5C2CC} 
\definecolor{MathBg}{HTML}{E6D4F2} 
\definecolor{ScienceBg}{HTML}{FBE0BC} 

\definecolor{mygreen}{RGB}{229, 245, 224}
\definecolor{mygray}{RGB}{242, 242, 242} 
\definecolor{myred}{RGB}{248, 230, 234}      

\definecolor{failbg}{RGB}{248, 230, 234}      
\definecolor{failframe}{RGB}{176, 36, 24}    
\definecolor{failbadge}{RGB}{225, 151, 168}      

\definecolor{successbg}{RGB}{239, 255, 229}   
\definecolor{successframe}{RGB}{34, 139, 34}  
\definecolor{successbadge}{RGB}{182, 200, 108}

\newcommand{\methodname}{Dr.\ MAS}

\definecolor{ThemePurple}{RGB}{95, 50, 105}
\definecolor{querycolor}{RGB}{112, 112, 112} 
\definecolor{verifycolor}{RGB}{117, 70, 126}
\definecolor{searchcolor}{RGB}{217, 140, 70}
\definecolor{solvercolor}{RGB}{217, 140, 70}

\hypersetup{
  colorlinks=true,
  linkcolor=deepred,
  citecolor=ThemePurple,
  urlcolor=ThemePurple
}

\usepackage{booktabs}
\usepackage{multirow}
\usepackage{makecell}
\usepackage{array}
\usepackage{siunitx}

\definecolor{TableBG}{RGB}{238, 240, 242}
\definecolor{DeltaUp}{RGB}{117, 70, 126}
\definecolor{DeltaDown}{RGB}{217, 140, 70}
\definecolor{DeltaZero}{RGB}{150, 150, 150}

\newcommand{\gup}[1]{\textcolor{DeltaUp}{\uparrow\!#1}}
\newcommand{\gdown}[1]{\textcolor{DeltaDown}{\downarrow\!#1}}
\newcommand{\gzero}[1]{\textcolor{DeltaZero}{\ensuremath{#1}}}

\newcommand{\valdelta}[2]{%
  \ensuremath{\text{#1}_{\smash{\scriptsize #2}}}%
}

\newtcolorbox{promptbox}[1][]{
  colback=gray!5!white,        
  colframe=black!75!black,     
  boxrule=0.3mm,               
  arc=3mm,                     
  auto outer arc=true,         
  width=\linewidth,            
  fontupper=\footnotesize,            
  #1                           
}

\tcbset{
    agentstyle/.style={
        colback=gray!5!white,          
        boxrule=0.3mm,                 
        arc=3mm,                       
        auto outer arc=true,           
        fontupper=\small               
    }
}

\newtcolorbox{verifierbox}[1][]{
    agentstyle,
    colframe=gray!60!verifycolor, 
    title=\textbf{Verifier Agent},
    #1 
}

\newtcolorbox{searchbox}[1][]{
    agentstyle,
    colframe=gray!60!searchcolor, 
    title=\textbf{Search Agent},
    #1
}

\newtcolorbox{answerbox}[1][]{
    agentstyle,
    colframe=teal!75!black,       
    title=\textbf{Answer Agent},
    #1
}

\newtcolorbox{solverbox}[1][]{
    agentstyle,
    colframe=gray!60!solvercolor, 
    title=\textbf{Solver Agent},
    #1
}

\newtcolorbox{envbox}[1][]{
    agentstyle,
    colframe=gray!60!querycolor,
    title=#1, 
}

\definecolor{alertred}{RGB}{217, 140, 70}
\definecolor{safeblue}{RGB}{150, 90, 165}
\definecolor{boxborder}{RGB}{60, 60, 60}

\tcbset{
    formulabox/.style={
        enhanced,
        colback=white,
        colframe=boxborder,
        boxrule=1.2pt,
        arc=4pt,
        top=12pt, bottom=12pt, left=10pt, right=10pt,
        halign=center
    }
}

\usepackage[utf8]{inputenc} 
\usepackage[T1]{fontenc}    
\usepackage{hyperref}       
\usepackage{url}            
\usepackage{booktabs}       
\usepackage{amsfonts}       
\usepackage{nicefrac}       
\usepackage{microtype}      
\usepackage{float}          
\usepackage{xcolor}         

\title{\methodname{}: Stable Reinforcement Learning for Multi-Agent LLM Systems}

\author{%
Lang Feng$^{1}$ \quad Longtao Zheng$^{1}$ \quad Shuo He$^{1}$ \quad Fuxiang Zhang$^{1}$ \quad Bo An$^{1}$\thanks{Corresponding author}\\
$^1$Nanyang Technological University, Singapore \\
\texttt{lang005@e.ntu.edu.sg}, \texttt{boan@ntu.edu.sg} \\[0.3em]
Project Page: \texttt{\url{https://langfengq.github.io/DrMAS}} \\
}

\begin{document}

\maketitle

\begin{abstract}

Multi-agent LLM systems improve reasoning and tool use via role specialization, yet reinforcement learning (RL) post-training for such systems remains unstable and underexplored. We theoretically pinpoint a key source of instability when extending group-based RL to cooperative multi-agent LLM systems: under GRPO-style optimization, a global normalization baseline can mismatch heterogeneous agents' reward distributions, inducing gradient-norm instability. Based on this finding, we propose \methodname{}, a simple and stable RL recipe for cooperative multi-agent LLM systems. Algorithmically, \methodname{} normalizes advantages per agent using each agent's own reward statistics, which calibrates gradient scales and dramatically stabilizes the training. Systemically, \methodname{} provides an end-to-end multi-agent RL framework with scalable orchestration, flexible per-agent LLM serving and optimization, and shared resource scheduling of actor backends. Unlike single-actor frameworks such as veRL, \methodname{} natively supports multiple heterogeneous LLMs over a unified GPU pool. Lifecycle-aware backend management and dynamic dispatch release inactive models' GPU memory and schedule active agents on demand, enabling hardware-efficient co-training. We evaluate \methodname{} on multi-agent math reasoning and multi-turn search with Qwen2.5 and Qwen3. \methodname{} achieves clear gains over vanilla GRPO (e.g., +5.6\% avg@16 and +4.6\% pass@16 on math, and +15.2\% avg@16 and +13.1\% pass@16 on search) while largely eliminating gradient spikes. It also remains highly effective under heterogeneous agent-model assignments while improving efficiency.
\end{abstract}

\begingroup
\setlength{\intextsep}{1pt}
\setlength{\abovecaptionskip}{1pt}
\begin{figure}[H]
   \centering
   \includegraphics[width=0.98\linewidth]{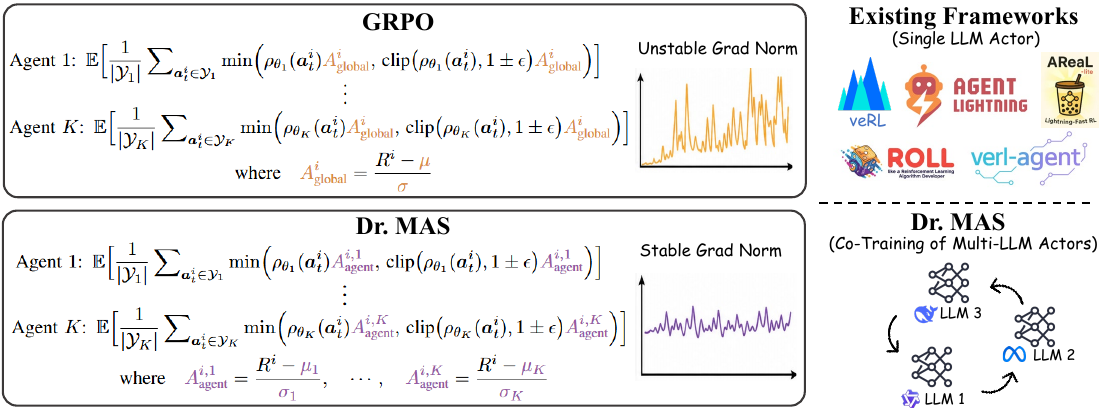}
   \caption{\textbf{Left} (algorithm): GRPO with global baseline can cause unstable gradient norm; \methodname{} with per-agent normalization stabilizes the training of MAS. \textbf{Right} (system): Existing RL frameworks mainly support single LLM actors, while \methodname{} enables co-training of multiple LLM actors.}
   \label{fig:introduction}
\end{figure}
\endgroup
\newcommand{\introfigureonfirstpage}{}

\ifdefined\introfigureonfirstpage
\else
\begin{figure*}[H]
   \begin{center}
   \includegraphics[width=0.98\linewidth]{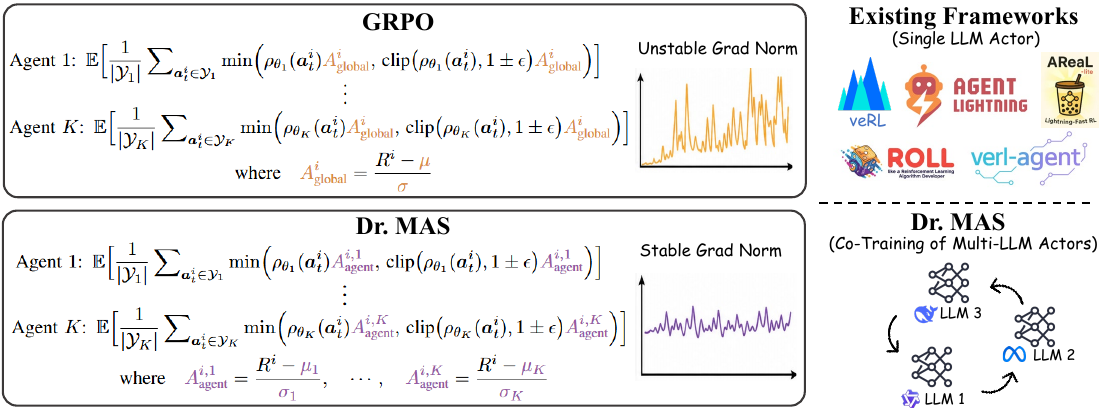}
   \end{center}
   \vspace{-0.1in}
   \caption{\textbf{Left} (algorithm): GRPO with global baseline can cause unstable gradient norm; \methodname{} with per-agent normalization stabilizes the training of MAS. \textbf{Right} (system): Existing RL frameworks mainly support single LLM actors, while \methodname{} enables co-training of multiple LLM actors.}
   \label{fig:introduction}
\end{figure*}
\fi
\section{Introduction}
Large Language Models (LLMs)~\citep{achiam2023gpt,team2023gemini,liu2024deepseek,yang2025qwen3} have shown strong reasoning abilities in tasks such as question answering and code generation. Yet many real-world applications, including complex information retrieval~\citep{chang2025main,zhang2025agentorchestra}, agentic software engineering~\citep{hong2024metagpt,qian2024chatdev}, and open-ended device control~\citep{wang2024mobile,tan2024cradle}, require multiple agents to coordinate over long horizons through planning, search, verification, and refinement. Organizing LLMs into Multi-Agent Systems (MAS), where each agent specializes in a subtask, has therefore become an important paradigm for solving complex real-world tasks~\citep{tran2025multi,zhang2025landscape}.

Reinforcement Learning (RL)~\citep{sutton2018reinforcement} has become central to LLM post-training, but RL for multi-agent LLM systems remains underexplored algorithmically and systemically. \textit{Algorithmically}, group-based methods like GRPO~\citep{shao2024deepseekmath} avoid learned critics by estimating a global baseline from final group scores. However, in MAS, agents are invoked with different frequencies and contexts and their reward distributions are heterogeneous; a single global baseline can therefore be miscalibrated for individual agents and induce unstable gradient scales. \textit{Systemically}, large-scale RL frameworks such as veRL~\citep{sheng2024hybridflow}, ROLL~\citep{wang2025reinforcement}, and AReaL~\citep{fu2025areal} provide efficient LLM training pipelines, but are largely built around a single LLM actor and lack native support for scalable multi-agent orchestration, heterogeneous agent-model assignment, and resource-efficient co-training of multiple LLMs.

In this work, we study RL post-training for cooperative multi-agent LLM systems from both perspectives. \textit{At the algorithmic level}, we \textit{theoretically} show that vanilla GRPO suffers from a baseline mismatch in MAS: a global advantage baseline over heterogeneous agent steps can deviate from each agent's reward statistics, inflate the per-agent gradient second moment, and trigger gradient-norm spikes. Based on this, we propose \methodname{}, a simple and stable recipe that replaces the global baseline with agent-wise baselines. By normalizing each agent's advantages using statistics from its own active timesteps, \methodname{} calibrates advantage scales across heterogeneous roles and removes the dominant gradient-inflation term.
\textit{At the systems level}, \methodname{} provides an end-to-end RL framework for multi-agent LLM post-training. It supports pluggable orchestration for diverse workflows, flexible agent-model assignment with optional LLM sharing or heterogeneous non-sharing, and per-agent optimization configurations. To make multi-LLM training practical, \methodname{} co-locates all LLM worker groups on a unified GPU resource pool, inspired by resource-bundle scheduling like Ray placement groups, and applies lifecycle-aware actor backend management: inactive models release GPU memory and KV caches, while active agents are dynamically scheduled. This avoids naive one-agent-one-GPU-pool allocation and enables efficient co-training of multiple LLM agents.

We evaluate \methodname{} on role-specialized MAS for math reasoning and multi-turn search with Qwen2.5~\citep{bai2025qwen2} and Qwen3~\citep{yang2025qwen3}, under both LLM-sharing and non-sharing settings. Across tasks and settings, \methodname{} consistently outperforms vanilla GRPO, improving math by +5.6\% avg@16 and +4.6\% pass@16, and search by +15.2\% avg@16 and +13.1\% pass@16, while largely eliminating gradient-norm spikes. It also remains effective under heterogeneous agent-model assignments, enabling smaller models for lower-level decisions while improving overall system efficiency.

Our contributions are:
\begin{itemize}[leftmargin=1.5em, topsep=1.5pt]
    \item \textbf{A theoretical foundation for agent-wise baselines in cooperative multi-agent LLM RL.}
    We show, to our knowledge for the first time, that GRPO-style optimization in multi-agent LLM systems can inflate the per-agent gradient second moment and cause instability, and we derive an agent-wise remedy that resolves this issue.
    \item \textbf{A scalable multi-agent LLM post-training framework beyond reward engineering.}
    \methodname{} integrates pluggable orchestration, flexible agent-model assignment, per-agent configurations, and lifecycle-aware GPU resource pooling, enabling hardware-efficient co-training of heterogeneous multi-agent LLMs under a unified RL pipeline.
\end{itemize}

\section{Related Work}
\textbf{Reinforcement Learning for LLMs.}~~
Beyond early alignment-focused approaches such as RLHF~\citep{ziegler2019fine,stiennon2020learning,ouyang2022training,rafailov2024direct}, recent work studies Reinforcement Learning from Verifiable Rewards (RLVR), which leverages automatically checkable signals (e.g., correctness in math or code) to improve LLM capabilities~\citep{zeng2025simplerl}.
Within this setting, group-based RL has emerged as a strong alternative to classical actor-critic algorithms like PPO~\citep{schulman2017proximal}. Techniques such as GRPO~\citep{shao2024deepseekmath}, RLOO~\citep{kool2019buy,ahmadian2024back},
Dr.~GRPO~\citep{liu2025understanding}, DAPO~\citep{yu2025dapo}, and GSPO~\citep{zheng2025group} aggregate multiple rollouts for the same query and perform relative comparisons within the group, thereby avoiding explicit value-function learning.
RLVR has also been extended to agentic, multi-turn settings where LLMs act as automatic agents~\citep{zhou2024archer,bai2024digirl,feng2025towards,wang2025spa,zhang2025agent,feng2026agentocr}. GRPO-style training has been widely applied to tool use~\citep{qian2025toolrl,xue2025simpletir}, OS control~\citep{lai2025computerrl}, and gaming~\citep{wang2025ragen}.
Further variants refine group construction or objectives, such as GiGPO~\citep{feng2025group} and ARPO~\citep{dong2025agentic}.

\textbf{Reinforcement Learning for Multi-Agent LLMs.}~~
Recent RL post-training has expanded from single-agent scenarios to \emph{learning coordination} in role-specialized, multi-turn multi-agent systems. Self-play training (e.g., SPIRAL~\citep{liu2025spiral}, and MARSHAL~\citep{yuan2025marshal}) use multi-turn dynamics to generate curricula and rewards with minimal manual labeling. However, these approaches are often confined to dyadic (two-agent) self-play scenarios. To ease deployment, Chain-of-Agents~\citep{li2025chain} distills multi-agent trajectories into a single agentic policy. Meanwhile, group-relative optimization has been extended to multi-agent settings~\citep{liu2025llm,chen2025heterogeneous,li2025flow,park2025maporl,wan2025rema,xue2025comas,mo2025multi,hong2025multi}, but these methods typically rely on heuristics and lack stability guarantees. \methodname{} distinguishes itself by theoretically identifying \emph{gradient-norm inflation} as the root cause of instability and proposing a simple yet rigorous agent-wise solution.

\textbf{Reinforcement Learning Infrastructure.}~~
As RL post-training scales, infrastructure has shifted toward optimizing the end-to-end \emph{rollout-train loop}, where throughput, scheduling, and variable-length generation dominate system efficiency. General-purpose stacks such as veRL~\citep{sheng2024hybridflow}, OpenRLHF~\citep{hu2024openrlhf}, ROLL~\citep{wang2025reinforcement}, slime~\citep{slime_github}, and AReaL~\citep{fu2025areal} increasingly provide modular pipeline abstractions and distributed execution to improve utilization under heavy sampling. As agentic use cases grow, frameworks increasingly emphasize multi-turn rollout and tool integration (e.g., verl-agent~\citep{feng2025group}, VerlTool~\citep{jiang2025verltool}, Agent-Lightning~\citep{luo2025agent}), with MARTI~\citep{marti2025} and PettingLLMs~\citep{zhao2025stronger} offering a practical multi-agent training interface. However, they either provide limited support for heterogeneous model assignments or lack a shared resource pool for efficient scheduling. Our \methodname{} addresses both to improve scalability and utilization in MAS post-training.

\section{Preliminaries}
\label{sec:background}
\subsection{Multi-Agent LLMs}
We consider a cooperative multi-agent LLM system consisting of $K$ distinct LLM agents ${\pi_{\theta_1}, \pi_{\theta_2}, \dots, \pi_{\theta_K}}$, each parameterized by its own LLM weights $\theta_k$. The agents jointly engage in solving complex tasks sampled from a distribution $x\in p(X)$. Each full interaction process (trajectory) produces a single outcome reward $R \in \mathbb{R}$ to indicate success or failure. During task completion, the agents’ joint interaction unfolds as a trajectory $\bm{\tau} = \{(\bm{s}_1, \bm{a}_1, k_1), (\bm{s}_2, \bm{a}_2, k_2), \dots, (\bm{s}_T, \bm{a}_T, k_T)\}$, where $\bm{s}_t$ denotes the conversational or contextual state (e.g., dialogue history, task prompt, or shared memory) at execution step $t$, $\bm{a}_t$ is the text output produced, and $k_t \in {1,\dots,K}$ denotes which LLM agent was active at step $t$. The active agent can change dynamically across the trajectory, for instance, in a hierarchical multi-agent framework, a high-level planner LLM may decide which sub-agent executes at each step. Hence, the index $k_t$ explicitly denotes the identity of the agent executing at each step. The execution LLM agent $k_t$ generates its output based on its policy $a_t \sim \pi_{\theta_{k_t}}(\cdot \mid s_t)$. Depending on the system design, the agents may share parameters (i.e., $\theta_1 = \dots = \theta_K$) , differing in role-specific prompts, enabling efficient adaptation under a unified LLM, or they may maintain distinct parameters ($\theta_i \neq \theta_j$) to specialize in heterogeneous sub-tasks, allowing diverse reasoning capabilities across agents.

\subsection{Group Relative Policy Optimization}
Group-based RL methods like Group Relative Policy Optimization (GRPO)~\citep{shao2024deepseekmath} optimize policies by comparing multiple rollouts generated from the same task instruction and normalizing their rewards within each group, thereby avoiding explicit value-function estimation. 
Formally, given a task instruction $x$, the multi-agent LLM system samples a set of $N$ trajectories
\begin{equation}
    \{\tau^i = (s_1^i, a_1^i, k_1^i, \dots, s_{T_i}^i, a_{T_i}^i, k_{T_i}^i)\}_{i=1}^N,
\end{equation}
where $k_t^i \in \{1,\dots,K\}$ denotes the active agent at step $t$ of trajectory $i$.  
Each trajectory $\tau^i$ receives a scalar terminal reward $R^i = R(\tau^i) \in \mathbb{R}$ that measures the overall quality of the generated outcome. 
The normalized advantage for each trajectory is computed using the group’s mean and standard deviation:
\begin{equation}
    A^i_{\text{global}} = \frac{R^i - \mu}{\sigma}, 
    \quad 
    \mu = \frac{1}{N}\sum\nolimits_{i=1}^N R^i, 
    \quad 
    \sigma^2 = \frac{1}{N}\sum\nolimits_{i=1}^N (R^i - \mu)^2.
\end{equation}
This advantage is then propagated to all agent outputs that contributed to the trajectory. Formally, we define the set of outputs of agent $k$ as $\mathcal{Y}_k = \{a_t^i \mid k_t^i = k\}$, i.e., the collection of all time steps $(i,t)$ across the group at which agent $k$ produces an action. Notably, agents are often invoked at different frequencies, which results in varying sample sizes $|\mathcal{Y}_k|$.
The RL objective for agent $k$ is then given by
\begin{equation}
    \mathcal{J}_k(\theta_k)
    = \mathbb{E}_{x\sim p(x)}\Big[
      \frac{1}{|\mathcal{Y}_k|}
      \sum\nolimits_{\bm{a}_t^i\in \mathcal{Y}_k}
      \min\Bigl(
        \rho_{\theta_k}(\bm{a}_t^{i}) A^i_{\text{global}},
        \,
        \text{clip}\bigl(\rho_{\theta_k}(\bm{a}_t^{i}), 1 \pm \epsilon\bigr) A^i_{\text{global}}
      \Bigr)
    \Big],
\end{equation}
where $\rho_{\theta_k}(\bm{a}_t^{i}) = \frac{\pi_{\theta_k}(\bm{a}_t^{i} \mid \bm{s}_t^{i})}{\pi_{{\theta_k}^{\text{old}}}(\bm{a}_t^i \mid \bm{s}_t^i)}$ is the importance sampling ratio. Here, we omit the KL-divergence regularization for notational brevity.
\section{Methodology}
In a multi-agent LLM system, different agents often specialize in
distinct functions (e.g., information retrieval vs. answer synthesis, high-level planning vs. low-level execution), and consequently can exhibit substantially different reward distributions. We find that using vanilla GRPO with the global baseline $(\mu,\sigma)$ for all agents can be suboptimal: some agents may consistently operate in reward distributions above the global mean, while others remain below it. This persistent bias in how advantages are normalized can introduce a deterministic shift in the effective advantages seen by each agent, which in turn can inflate gradient-estimator variance and destabilize training.

In this section, we introduce \methodname{} by (\textbf{1}) theoretically formalizing the instability and analyzing the second moment of the per-agent gradient under GRPO optimization (Section~\ref{sec:risk}); (\textbf{2}) proposing an agent-wise remedy that calibrates each agent's advantage using its own reward statistics, thereby improving the training stability (Section~\ref{sec:remedy}); and (\textbf{3}) describing a system framework that implements an efficient end-to-end RL training recipe for multi-agent LLM systems (Section~\ref{sec:framework}). The complete pseudo code of \methodname{} is provided in Appendix~\ref{app:pseudocode}.

\subsection{Risk of Gradient Norm Explosion}
\label{sec:risk}
To focus on how advantage normalization causes the instability in MAS training, we perform a theoretical analysis of the gradient norm. We base our analysis on the unclipped GRPO gradient (clipping and other regularization only further bound the update and are orthogonal to the gradient issue we study).
For each agent $k$, and for each step $(i,t)$ such that $k_t^i = k$, we define
the \emph{score function} as $z_{i,t}^{(k)} \triangleq \nabla_{\theta_k} \rho_{\theta_k}(\bm{a}_t^{i})$ and corresponding (unclipped) GRPO gradient contribution as
\begin{equation}
    \tilde{g}_k^{\mathrm{global}}
    \triangleq
    \frac{R^i - \mu}{\sigma}\,
    z_{i,t}^{(k)}.
\end{equation}
Here, $(\mu,\sigma)$ are the mean and standard deviation used by vanilla GRPO. We assume that each agent’s score function has a bounded second moment:
\begin{assumption}
\label{assump:score}
\emph{For each agent $k$, there exists a constant $C_k < \infty$ such that $0 \le \mathbb{E}\big[\,\|z_{i,t}^{(k)}\|^2\,\big] \le C_k$.}
\end{assumption}
Then we can express the second moment of the per-agent gradient as follows.
\begin{lemma}
\label{lem:global-second-moment}
Under Assumptions~\ref{assump:score}, for any agent $k$,
\begin{equation}
\nonumber
\mathbb{E}_{\bm{a}_t^i\sim \mathcal{Y}_k}\big[\|\tilde{g}_k^{\mathrm{global}}\|^2\big]
{=} \mathbb{E}_{\bm{a}_t^i\sim \mathcal{Y}_k}\big[\|z_{i,t}^{(k)}\|^2\big]\,
  \frac{\sigma_k^2 + (\mu_k - \mu)^2}{\sigma^2} + \Delta_k,
\end{equation}
where $\mu_k \triangleq \frac{1}{|\mathcal{Y}_k|}\sum_{\bm{a}_t^i\in \mathcal{Y}_k} R^i$, $\sigma_k^2 \triangleq \frac{1}{|\mathcal{Y}_k|}\sum_{\bm{a}_t^i\in \mathcal{Y}_k} (R^i - \mu_k)^2$
are the mean and variance when sampling
time steps uniformly from $\mathcal{Y}_k$ (i.e., when agent $k$ is active). $\Delta_k$ is a score-reward covariance correction term.
\end{lemma}
See Appendix~\ref{app:proof_lemma_global_second_moment} for the proof.
Lemma~\ref{lem:global-second-moment} separates the per-agent gradient norm into a dominant scaling factor and a residual covariance correction. The multiplicative factor $\bigl(\sigma_k^2+(\mu_k-\mu)^2\bigr)/\sigma^2$ grows when agent $k$ operates in a reward distribution whose mean is far from the global mean or agent $k$’s conditional reward variance is much larger than the global variance.
The term $\Delta_k$ captures the residual score-reward correlation. In large-scale LLM training, rewards are typically low-dimensional signals of final task quality (e.g., pass/fail for reasoning, correctness for coding), while $z_{i,t}^{(k)}$ depends mainly on the local token-level stochasticity of the policy. Empirically, their covariance is often much smaller than the dominant scaling factor $\mathbb{E}_k[\|z_{i,t}^{(k)}\|^2](\sigma_k^2 + (\mu_k - \mu)^2)/\sigma^2$. This decomposition reveals the intrinsic instability of global normalization of GRPO in heterogeneous multi-agent training: a large deviation in the dominant scaling factor can inflate the gradient and lead to unstable updates. We formalize this phenomenon below.
\begin{proposition}[Gradient-Norm Inflation]
\label{prop:gradient-inflation}
As either the normalized mean deviation
$|\mu_k - \mu| / \sigma$ or the normalized variance ratio
$\sigma_k^2 / \sigma^2$ becomes large, the second moment of
$\tilde{g}_k^{\mathrm{global}}$ grows at least linearly. Consequently, along any training process for which there exists a sequence of iterations indexed by $m$ such that
$$
\frac{\sigma_{k,m}^2 + (\mu_{k,m} - \mu_m)^2}{\sigma_m^2} \to \infty, \quad  \mathbb{E}\bigl[\|\tilde{g}_m^{\mathrm{global}}\|^2\bigr] \to \infty,
$$
where $\tilde{g}_{m}^{\mathrm{global}} = (\tilde{g}_{1,m}^{\mathrm{global}},\dots,\tilde{g}_{K,m}^{\mathrm{global}})$ stacking all LLM agents' gradients.
\end{proposition}
The proof is provided in Appendix~\ref{app:proof_prop_gradient_inflation}. Proposition~\ref{prop:gradient-inflation} demonstrates that gradient-norm inflation can be triggered by \emph{any} agent whose reward distribution is poorly aligned with the global baseline. In practice, the gradient norms in such cases typically do not reach mathematical infinity. However, they often grow large enough and trigger severe gradient spikes, thus destabilizing the training process of the entire multi-agent LLM system.

\subsection{Agent-Wise Remedy}
\label{sec:remedy}
Fortunately, Proposition~\ref{prop:gradient-inflation} suggests a straightforward and effective remedy: calibrating each agent’s advantages using reward statistics computed exclusively on the steps where that agent is active.
Specifically, we replace the global baseline $(\mu,\sigma)$ with $(\mu_k,\sigma_k)$, which ensures that $(\sigma_k^2 + (\mu_k - \mu)^2)/\sigma^2=1$. In practice, this corresponds to normalizing each agent’s reward using its own empirical mean and variance: 
\begin{equation}
    A^{i,k}_{\mathrm{agent}} = \frac{R^i - \mu_k}{\sigma_k},
\end{equation}
where $\mu_k \triangleq \frac{1}{|\mathcal{Y}_k|}\sum_{\bm{a}_t^i\in \mathcal{Y}_k} R^i$ and $\sigma_k^2 \triangleq \frac{1}{|\mathcal{Y}_k|}\sum_{\bm{a}_t^i\in \mathcal{Y}_k} (R^i - \mu_k)^2$. Since \(\mathcal{Y}_k\) is defined over active steps rather than unique
trajectories, the statistics above are action-weighted: if agent \(k\)
appears multiple times in the same trajectory, the trajectory-level reward
\(R^i\) contributes once for each active step. This matches the objective,
which also averages gradients over active agent outputs. Therefore, an analysis analogous to
Lemma~\ref{lem:global-second-moment} yields
\begin{equation}
    \mathbb{E}_{\bm{a}_t^i\sim \mathcal{Y}_k}\bigl[\|\tilde{g}_k^{\mathrm{agent}}\|^2\bigr]=\mathbb{E}_{\bm{a}_t^i\sim \mathcal{Y}_k}\bigl[\|z_{i,t}^{(k)}\|^2\bigr] +\Delta_k^{\mathrm{agent}},
\end{equation}
where $\tilde{g}_k^{\mathrm{agent}}=\frac{R^i - \mu_k}{\sigma_k}\, z_{i,t}^{(k)}$. Thus, under agent-wise normalization, the second moment of each agent’s gradient is bounded by its own score statistics. Crucially, this effect is inherently \emph{multi-agent}: as the number of specialized agents increases and their roles become more heterogeneous, a single global baseline is increasingly likely to be badly aligned with some agents, leading to gradient norm explosion. A simple agent-wise remedy, by adapting to each agent’s own statistics, achieves keeping all gradients in a comparable, well-conditioned range, while still enabling cooperative optimization of the overall multi-agent LLM system.

\begin{figure*}[t]
   \begin{center}
   \includegraphics[width=1.0\linewidth]{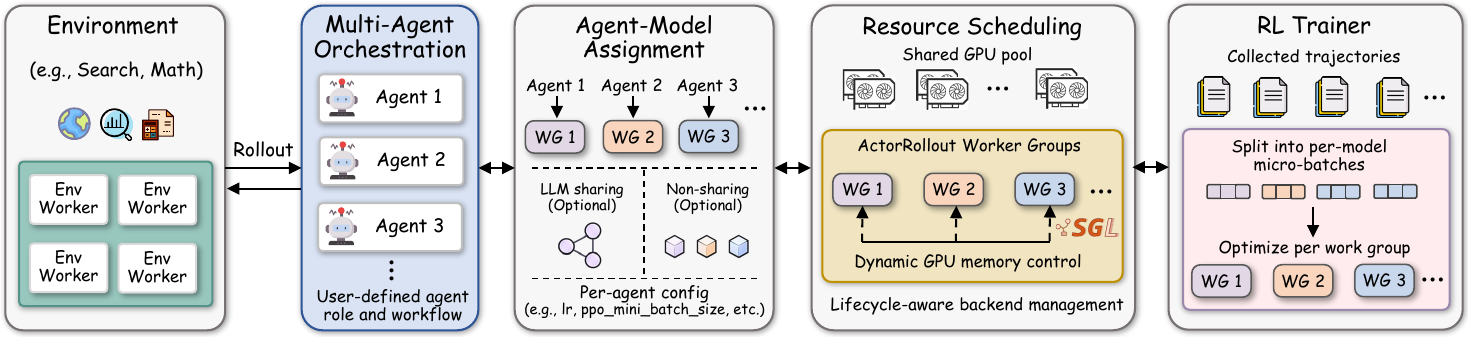}
   \end{center}
   \vspace{-0.1in}
   \caption{Overview of multi-agent LLM RL framework. A multi-agent orchestrator manages distributed rollouts, agents are mapped to LLM worker groups with optional LLM  sharing, and a shared resource pool schedules actor backends for efficient inference and per-model optimization.}
   \label{fig:framework}
\end{figure*}

\subsection{Framework for Multi-Agent LLM RL}
\label{sec:framework}

We next present a unified framework for end-to-end RL post-training of cooperative multi-agent LLM systems. As shown in Figure~\ref{fig:framework}, \methodname{} supports flexible multi-agent orchestration, agent--model assignment, per-agent optimization, and efficient GPU utilization within a single training pipeline.

\textbf{Multi-Agent Orchestration.}~
A multi-agent trajectory collector coordinates interactions between agents and the environment. It delegates rollout execution to a pluggable orchestra, which defines agent roles and execution flows such as sequential collaboration, conditional routing, and per-agent active masking. At each step, the orchestra selects the active agent based on the current state and previous agent outputs.

\textbf{Agent-Model Assignment.}~
A core assignment logic maps logical agents ($1,\dots,K$) to physical LLM worker groups (\texttt{wg\_id}). In non-shared settings, each agent $k$ is assigned a distinct worker group (e.g., 7B and 3B models). Conversely, in shared settings, all agents configured with the same model are mapped to a single, shared worker group, allowing joint training and inference while reusing model weights. 

\textbf{Per-Agent Configuration.}~
\methodname{} supports agent-specific training hyperparameters for granular control. This allows configurations like \texttt{actor.optim.lr} to be specified on a per-agent basis. Our system injects the $k$-th hyperparameter set into the configuration for agent $k$, which is then attached to its corresponding LLM work group. A runtime check ensures that all agents sharing the same worker group utilize identical configurations.

\textbf{Shared GPU Resource Pooling and Scheduling.}~
To avoid inefficient one-agent-one-GPU allocation, \methodname{} co-locates all LLM worker groups in a unified GPU resource pool. All \texttt{ActorRollout} backends are scheduled through Ray placement groups and served by \texttt{sglang}~\citep{zheng2024sglang}. During rollout, requests are routed by \texttt{agent\_to\_wg\_mapping} to the corresponding backend \texttt{actor\_rollout\_wg[wg\_id]}. Each backend exposes \texttt{resume()} and \texttt{offload()} controls: active backends load weights and KV caches onto GPU, while inactive ones release GPU memory. This enables multiple shared or non-shared LLM agents to be co-trained over the same GPU pool without requiring all models to remain GPU-resident simultaneously.

\textbf{Per-Worker-Group Optimization.}~
After rollout collection, the trainer partitions the aggregated batch $\mathcal{B}$ by worker-group ID: $\mathcal{B}_{g}=\{\,\tau \in \mathcal{B}\mid \texttt{wg\_id}(\tau)=g\,\}$.
Policy updates are then performed for each worker group, ensuring that gradients from an agent's trajectories only update its designated LLM backend.
\section{Experiment}
In this section, we evaluate \methodname{} on two multi-agent orchestrations: a two-agent loop pipeline for math reasoning and a three-agent hierarchical pipeline for multi-turn search (Figure~\ref{fig:orchestration}), under both LLM sharing and non-sharing settings. Specifically, we aim to demonstrate: (\textbf{1}) the consistent performance gains of \methodname{} over vanilla GRPO; (\textbf{2}) more stable training dynamics and smoother gradient norms of \methodname{}; (\textbf{3}) the individual contribution of each normalization component via a detailed ablation study; and (\textbf{4}) the practical efficiency and compatibility of \methodname{} when applied to heterogeneous agent-model assignments.

\begin{table*}[t]
\centering
\caption{Math results on Qwen3-4B/8B. We report the avg@16 and pass@16 of single-agent training with GRPO and multi-agent training under LLM sharing/non-sharing, using vanilla GRPO and \methodname{}. Subscripts for \methodname{} denote $\Delta$ over the vanilla GRPO under the same multi-agent setting.}
\label{tab:main_result_math}
\resizebox{1.0\textwidth}{!}{
\begin{tabular}{
l
|cc
|cc
>{\columncolor{TableBG}}c >{\columncolor{TableBG}}c
|cc
>{\columncolor{TableBG}}c >{\columncolor{TableBG}}c
}
\toprule
\multirow{3}{*}{\textbf{Benchmark}} &
\multicolumn{2}{c|}{\textbf{Single-Agent}} &
\multicolumn{4}{c|}{\textbf{Multi-Agent \& LLM Sharing}} &
\multicolumn{4}{c}{\textbf{Multi-Agent \& LLM Non-Sharing}} \\
& \multicolumn{2}{c|}{GRPO} &
\multicolumn{2}{c}{GRPO} & \multicolumn{2}{c|}{\methodname{}} &
\multicolumn{2}{c}{GRPO} & \multicolumn{2}{c}{\methodname{}} \\
\rowcolor{white!4}
& avg@16 & pass@16 &
avg@16 & pass@16 & avg@16 & pass@16 &
avg@16 & pass@16 & avg@16 & pass@16 \\
\midrule

\rowcolor{black!4}\multicolumn{11}{c}{\textit{Qwen3-4B}} \\
\midrule
AIME'24    & 38.8 & 63.3 & 39.3 & 64.0 & \valdelta{39.3}{\gzero{0.0}\phantom{00}} & \valdelta{63.3}{\gdown{0.7}\phantom{0}} & 42.7 & 73.3 & \valdelta{46.9}{\gup{4.2}} & \valdelta{80.0}{\gup{6.7}} \\
AIME'25    & 33.1 & 56.7 & 31.4 & 53.3 & \valdelta{38.1}{\gup{6.7}\phantom{0}} & \valdelta{63.3}{\gup{10.0}} & 35.6 & 63.3 & \valdelta{38.1}{\gup{2.5}} & \valdelta{66.7}{\gup{3.4}} \\
AMC'23     & 83.5 & 95.0 & 85.6 & 95.0 & \valdelta{87.3}{\gup{1.7}\phantom{0}} & \valdelta{95.0}{\gzero{0.0}\phantom{00}} & 83.5 & 95.0 & \valdelta{89.5}{\gup{6.0}} & \valdelta{97.5}{\gup{2.5}} \\
MATH500    & 89.0 & 94.2 & 89.5 & 96.2 & \valdelta{90.5}{\gup{1.0}\phantom{0}} & \valdelta{96.0}{\gdown{0.2}\phantom{0}} & 89.6 & 95.0 & \valdelta{92.4}{\gup{2.8}} & \valdelta{97.0}{\gup{2.0}} \\
Minerva    & 37.9 & 49.6 & 37.5 & 50.0 & \valdelta{40.9}{\gup{3.4}\phantom{0}} & \valdelta{53.3}{\gup{3.3}\phantom{0}} & 37.5 & 50.7 & \valdelta{39.0}{\gup{1.5}} & \valdelta{51.5}{\gup{0.8}} \\
Olympiad   & 53.3 & 66.5 & 57.6 & 65.6 & \valdelta{58.2}{\gup{0.6}\phantom{0}} & \valdelta{68.6}{\gup{3.0}\phantom{0}} & 56.3 & 68.9 & \valdelta{60.9}{\gup{4.6}} & \valdelta{73.6}{\gup{4.7}} \\
\textbf{Average} & 55.9 & 70.9 & 56.8 & 70.7 & \valdelta{59.0}{\gup{2.2}\phantom{0}} & \valdelta{73.2}{\gup{2.6}\phantom{0}} & 57.5 & 74.4 & \valdelta{61.1}{\gup{3.6}} & \valdelta{77.7}{\gup{3.3}} \\

\midrule
\rowcolor{black!4}\multicolumn{11}{c}{\textit{Qwen3-8B}} \\
\midrule
AIME'24    & 36.0 & 67.3 & 42.7 & 66.7 & \valdelta{54.8}{\gup{12.1}} & \valdelta{80.0}{\gup{13.3}} & 42.9 & 70.0 & \valdelta{44.6}{\gup{1.7}} & \valdelta{73.3}{\gup{3.3}} \\
AIME'25    & 32.7 & 50.0 & 31.4 & 53.3 & \valdelta{39.4}{\gup{8.0}\phantom{0}} & \valdelta{70.0}{\gup{16.7}} & 31.8 & 53.3 & \valdelta{41.5}{\gup{9.7}} & \valdelta{56.7}{\gup{3.4}} \\
AMC'23     & 87.0 & 95.0 & 87.3 & 95.0 & \valdelta{88.9}{\gup{1.6}\phantom{0}} & \valdelta{97.5}{\gup{2.5}\phantom{0}} & 86.1 & 95.0 & \valdelta{87.5}{\gup{1.4}} & \valdelta{95.0}{\gzero{0.0}\phantom{0}} \\
MATH500    & 89.9 & 94.8 & 89.6 & 96.2 & \valdelta{91.3}{\gup{1.7}\phantom{0}} & \valdelta{96.0}{\gdown{0.2}\phantom{0}} & 90.5 & 96.6 & \valdelta{90.7}{\gup{0.2}} & \valdelta{96.2}{\gdown{0.4}} \\
Minerva    & 36.0 & 46.7 & 37.5 & 50.0 & \valdelta{39.9}{\gup{2.4}\phantom{0}} & \valdelta{49.6}{\gdown{0.4}\phantom{0}} & 39.2 & 50.7 & \valdelta{40.9}{\gup{1.7}} & \valdelta{54.0}{\gup{3.3}} \\
Olympiad   & 57.9 & 67.5 & 58.2 & 71.4 & \valdelta{59.3}{\gup{1.1}\phantom{0}} & \valdelta{72.4}{\gup{1.0}\phantom{0}} & 58.2 & 67.6 & \valdelta{59.0}{\gup{0.8}} & \valdelta{70.2}{\gup{2.6}} \\
\textbf{Average} & 56.6 & 70.2 & 57.8 & 72.1 & \valdelta{62.3}{\gup{4.5}\phantom{0}} & \valdelta{77.6}{\gup{5.5}\phantom{0}} & 58.1 & 72.2 & \valdelta{60.7}{\gup{2.6}} & \valdelta{74.2}{\gup{2.0}} \\
\bottomrule
\end{tabular}
}
\end{table*}
\subsection{Math}
\textbf{Setup.}~
We first evaluate \methodname{} on challenging mathematical reasoning tasks using a two-agent architecture with a \emph{solver agent} and a \emph{verifier agent}: the solver proposes or refines solutions, while the verifier accepts the solution or requests another refinement round. The detailed workflow are provided in Appendix~\ref{app:workflows}. We use Qwen3-4B/8B~\citep{yang2025qwen3} as the LLM policy for each agent and evaluate both the shared-LLM and non-shared settings.
For training, we adopt the training corpus from DAPO-Math~\citep{yu2025dapo}, which consists of diverse math problems paired with verifiable solutions and reward signals. The rollout group size is set to 8. For evaluation, we report the avg@16 and pass@16 results on a suite of competitive benchmarks: AIME'24, AIME'25, AMC'23, MATH500~\citep{hendrycks2021measuring}, Minerva, and OlympiadBench~\citep{he2024olympiadbench}. All other experimental details are available in Appendix~\ref{app:experimental_details_math}.

\textbf{Results.}~
As shown in Table~\ref{tab:main_result_math}, \methodname{} improves over vanilla GRPO under both LLM sharing and LLM non-sharing. While applying GRPO directly to the multi-agent setting can reach decent average scores, the gains are not always consistent across benchmarks, and some hard splits may not improve. This suggests that a single global normalization of GRPO can make multi-agent training less reliable. In contrast, \methodname{} uses per-agent normalization to keep each agent’s update on a similar scale, which leads to more consistent improvements across datasets and settings (an overall increase of 5.6\% in avg@16 and 4.6\% in pass@16). For Qwen3-4B, \methodname{} improves the performance under the sharing setting from 56.8/70.7 to 59.0/73.2 and improves the non-sharing setting from 57.5/74.4 to 61.1/77.7. The pronounced gain in the non-shared setting suggests that when agents possess independent parameters, their behavioral distributions diverge more significantly, making \methodname{}'s agent-specific calibration even more critical. Similarly, for Qwen3-8B, we observe strong gains in both configurations. The most significant improvements occur on the challenging AIME benchmarks (e.g., 42.7/66.7 $\rightarrow$ 54.8/80.0 on AIME’24), demonstrating that high-variance gradients from GRPO can easily disrupt the learning of fragile, long-horizon reasoning chains. \methodname{} promotes stable convergence, allowing agents to robustly learn the precise, multi-stage deductions.

\begin{table*}[t]
\centering
\caption{Multi-turn search QA results on Qwen2.5-3B/7B. We report the avg@16 and pass@16 of single-agent training with GRPO and multi-agent training under LLM sharing/non-sharing, using vanilla GRPO and \methodname{}. Subscripts for \methodname{} denote $\Delta$ over the vanilla GRPO under the same multi-agent setting.}
\label{tab:main_result_search}
\resizebox{1.0\textwidth}{!}{
\begin{tabular}{
l|cc|cc>{\columncolor{TableBG}}c >{\columncolor{TableBG}}c|cc >{\columncolor{TableBG}}c >{\columncolor{TableBG}}c
}
\toprule
\multirow{3}{*}{\textbf{Benchmark}} &
\multicolumn{2}{c|}{\textbf{Single-Agent}} &
\multicolumn{4}{c|}{\textbf{Multi-Agent \& LLM Sharing}} &
\multicolumn{4}{c}{\textbf{Multi-Agent \& LLM Non-Sharing}} \\
& \multicolumn{2}{c|}{GRPO} &
\multicolumn{2}{c}{GRPO} & \multicolumn{2}{c|}{\methodname{}} &
\multicolumn{2}{c}{GRPO} & \multicolumn{2}{c}{\methodname{}} \\
\rowcolor{white!4}
& avg@16 & pass@16 &
avg@16 & pass@16 & avg@16 & pass@16 &
avg@16 & pass@16 & avg@16 & pass@16 \\
\midrule

\rowcolor{black!4}\multicolumn{11}{c}{\textit{Qwen2.5-3B}}\\
\midrule
NQ         & 40.6 & 54.7 & 41.0 & 59.0 & \valdelta{43.8}{\gup{2.8}} & \valdelta{58.5}{\gdown{0.5}} & 43.8 & 54.5 & \valdelta{44.6}{\gup{0.8}\phantom{0}} & \valdelta{58.1}{\gup{3.6}\phantom{0}}\\
TriviaQA   & 58.1 & 68.8 & 57.9 & 68.4 & \valdelta{61.7}{\gup{3.8}} & \valdelta{70.1}{\gup{1.7}} & 60.6 & 70.8 & \valdelta{61.1}{\gup{0.5}\phantom{0}} & \valdelta{71.7}{\gup{0.9}\phantom{0}}\\
PopQA      & 44.2 & 49.6 & 43.2 & 58.0 & \valdelta{45.0}{\gup{1.8}} & \valdelta{57.6}{\gdown{0.4}} & 45.6 & 54.5 & \valdelta{46.5}{\gup{0.9}\phantom{0}} & \valdelta{57.4}{\gup{2.9}\phantom{0}}\\
HotpotQA   & 31.8 & 40.9 & 32.5 & 48.0 & \valdelta{33.3}{\gup{0.8}} & \valdelta{51.2}{\gup{3.2}} & 32.5 & 45.2 & \valdelta{35.3}{\gup{2.8}\phantom{0}} & \valdelta{51.1}{\gup{5.9}\phantom{0}}\\
2Wiki      & 29.9 & 43.7 & 33.7 & 64.0 & \valdelta{34.1}{\gup{0.4}} & \valdelta{64.0}{\gzero{0.0}} & 29.2 & 48.9 & \valdelta{34.9}{\gup{5.7}\phantom{0}} & \valdelta{60.2}{\gup{11.3}}\\
MuSiQue    & 7.9 & 14.6 & 9.1 & 26.5 & \valdelta{10.2}{\gup{1.1}} & \valdelta{25.8}{\gdown{0.7}} & 8.6 & 19.2 & \valdelta{10.4}{\gup{1.8}\phantom{0}} & \valdelta{26.1}{\gup{6.9}\phantom{0}}\\
Bamboogle  & 15.3 & 27.2 & 26.4 & 46.4 & \valdelta{28.6}{\gup{2.2}} & \valdelta{49.6}{\gup{3.2}} & 21.0 & 33.6 & \valdelta{25.4}{\gup{4.4}\phantom{0}} & \valdelta{46.4}{\gup{12.8}}\\
\textbf{Average} & 32.5 & 42.8 & 34.8 & 52.9 & \valdelta{36.7}{\gup{1.8}} & \valdelta{53.8}{\gup{0.9}} & 34.5 & 46.7 & \valdelta{36.9}{\gup{2.4}\phantom{0}} & \valdelta{53.0}{\gup{6.3}\phantom{0}}\\

\midrule
\rowcolor{black!4}\multicolumn{11}{c}{\textit{Qwen2.5-7B}}\\
\midrule
NQ         & 46.4 & 57.6 & 45.2 & 60.0 & \valdelta{47.4}{\gup{2.2}} & \valdelta{60.7}{\gup{0.7}} & 27.1 & 39.0 & \valdelta{47.7}{\gup{20.6}} & \valdelta{59.5}{\gup{20.5}}\\
TriviaQA   & 63.1 & 72.4 & 63.9 & 70.9 & \valdelta{63.1}{\gdown{0.8}} & \valdelta{71.2}{\gup{0.3}} & 53.1 & 64.4 & \valdelta{63.4}{\gup{10.3}} & \valdelta{72.7}{\gup{8.3}\phantom{0}}\\
PopQA      & 47.2 & 53.6 & 43.9 & 55.0 & \valdelta{45.9}{\gup{2.0}} & \valdelta{57.3}{\gup{2.3}} & 20.7 & 27.9 & \valdelta{46.7}{\gup{26.0}} & \valdelta{57.8}{\gup{29.9}}\\
HotpotQA   & 43.0 & 55.1 & 40.3 & 55.0 & \valdelta{42.5}{\gup{2.2}} & \valdelta{56.0}{\gup{1.0}} & 24.4 & 36.2 & \valdelta{44.0}{\gup{19.6}} & \valdelta{57.5}{\gup{21.3}}\\
2Wiki      & 40.6 & 61.6 & 41.6 & 67.8 & \valdelta{42.0}{\gup{0.4}} & \valdelta{67.1}{\gdown{0.7}} & 30.3 & 51.2 & \valdelta{45.4}{\gup{15.1}} & \valdelta{68.1}{\gup{16.9}}\\
MuSiQue    & 17.8 & 34.6 & 15.2 & 31.7 & \valdelta{16.7}{\gup{1.5}} & \valdelta{32.1}{\gup{0.4}} & 8.3 & 18.1 & \valdelta{19.4}{\gup{11.1}} & \valdelta{34.9}{\gup{16.8}}\\
Bamboogle  & 36.7 & 54.4 & 40.1 & 58.4 & \valdelta{40.1}{\gzero{0.0}\phantom{0}} & \valdelta{59.2}{\gup{0.8}} & 31.9 & 46.4 & \valdelta{39.8}{\gup{7.9}\phantom{0}} & \valdelta{57.6}{\gup{11.2}}\\
\textbf{Average} & 42.1 & 55.6 & 41.5 & 57.0 & \valdelta{42.5}{\gup{1.1}} & \valdelta{57.7}{\gup{0.7}} & 28.0 & 40.5 & \valdelta{43.8}{\gup{15.8}} & \valdelta{58.3}{\gup{17.8}}\\
\bottomrule
\end{tabular}
}
\end{table*}

\subsection{Multi-Turn Search}
\textbf{Setup.}~
We then evaluate \methodname{} on the multi-turn search tool-calling task using a three-agent hierarchy with a \emph{verifier agent}, a \emph{search agent}, and an \emph{answer agent}: the verifier decides whether more evidence is needed, invokes the search agent when necessary, and calls the answer agent to synthesize the final response.
The detailed workflow are provided in Appendix~\ref{app:workflows}. We use Qwen2.5-3B/7B~\citep{yang2024qwen2} as the LLM policy for each agent and evaluate both the shared-LLM and non-shared settings.
Our experimental setup follows Search-R1~\citep{jin2025search}. We employ E5~\citep{wang2022text} as the retriever. The rollout group size is set to 5 and the max turn is set to 4. For evaluation, we consider both single-hop QA benchmarks (NQ~\citep{kwiatkowski2019natural}, TriviaQA~\citep{joshi2017triviaqa}, PopQA~\citep{mallen2022not}) and multi-hop QA benchmarks (HotpotQA~\citep{yang2018hotpotqa}, 2WikiMultiHopQA~\citep{ho2020constructing}, MuSiQue~\citep{trivedi2022musique}, Bamboogle~\citep{press2022measuring}) and report the avg@16 and pass@16 results. For training, we use the mixture of NQ and HotpotQA. All other experimental details are available in Appendix~\ref{app:experimental_details_search}.


\textbf{Results.}~
As shown in Table~\ref{tab:main_result_search}, the instability shows up more clearly in multi-turn search, as errors can snowball across tool calls and across agents. In this setting, vanilla GRPO is especially risky when LLMs are not shared, since each agent can drift and the same global scaling may no longer match their learning dynamics. A clear example of this failure occurs with Qwen2.5-7B (non-sharing), where vanilla GRPO learns to avoid calling search agents entirely due to high gradient norms, leading to a severe performance drop (28.0/40.5). In contrast, \methodname{} effectively mitigates this risk and yields consistent improvements, with an overall increase of 15.2\% in avg@16 and 13.1\% in pass@16. Notably, \methodname{} restores the performance of Qwen2.5-7B to 43.8/58.3, a result that not only far exceeds the vanilla baseline but also surpasses both the single-agent baseline and the shared-LLM setting. This trend, consistent with our results in math tasks, highlights that LLM non-sharing can be hurt badly without proper stabilization, and reducing training noise at the agent level becomes crucial.

\begin{figure*}[t]
   \begin{center}
   \includegraphics[width=1.0\linewidth]{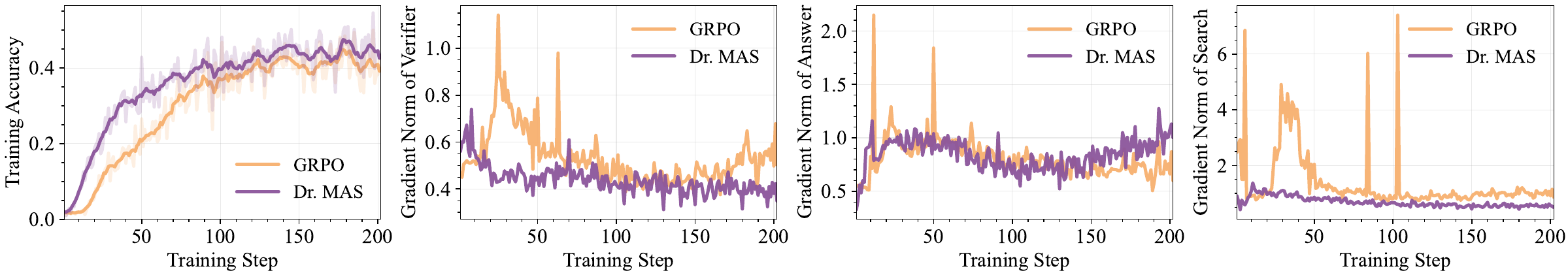}
   \end{center}
   \vspace{-0.1in}
    \caption{Comparison of training accuracy and gradient norm between GRPO and \methodname{}. The results are recorded during multi-agent RL post-training for three-agent search orchestration under LLM non-sharing (Qwen2.5-3B).}
   \label{fig:grad_norm_search}
\end{figure*}
\subsection{Gradient-Norm Instability}
In this part, we investigate the gradient-norm instability by tracking the training accuracy and per-agent gradient norms during RL post-training for the three-agent search orchestration (see Appendix~\ref{app:gradient_norm_math} for math results). As shown in Figure~\ref{fig:grad_norm_search}, vanilla GRPO induces frequent, high-magnitude gradient norm spikes. The \emph{search} agent has the largest spikes (reaching very high values early and again around the middle of training), the \emph{answer} agent also exhibits large spikes at the beginning, and the \emph{verifier} shows noticeable peaks as well. These spikes mean that some steps produce unusually large updates, which makes training noisy and harder to control.

\methodname{} mitigates this failure via the agent-wise remedy that normalizes advantages per agent, keeping per-agent update scales better calibrated. As illustrated in Figure~\ref{fig:grad_norm_search}, \methodname{} keeps the gradient norms of all three agents much smoother and at a lower level throughout training, therefore achieving notable performance gains in Tables~\ref{tab:main_result_math} and \ref{tab:main_result_search}.

\subsection{Ablation Study}\label{sec:ablation}
\begin{wraptable}{r}{0.5\textwidth}
\small
\centering
\caption{Ablation study of different advantage normalization configurations on the search task. We report the avg@16 and pass@16 across all datasets. Subscripts denote $\Delta$ over the vanilla GRPO.}
\label{tab:ablation}
\vspace{0.05in}
\resizebox{\linewidth}{!}{
\begin{tabular}{l|ccc>{\columncolor{TableBG}}c}
\toprule
\multirow{2}{*}{\textbf{Metric}} & \multicolumn{4}{c}{\textbf{Normalization Configuration}} \\
\cmidrule(l){2-5}
\rowcolor{white!4} & $(\mu, \sigma)$ & $(\mu_k, \sigma)$ & $(\mu, \sigma_k)$ & $(\mu_k, \sigma_k)$ \\
\midrule
avg@16 & 28.0 & \valdelta{39.1}{\gup{11.1}} & \valdelta{42.9}{\gup{14.9}} & \valdelta{43.8}{\gup{15.8}} \\
pass@16 & 40.5 & \valdelta{53.5}{\gup{13.0}} & \valdelta{57.6}{\gup{17.1}} & \valdelta{58.3}{\gup{17.8}} \\
\bottomrule
\end{tabular}
}
\vspace{-0.15in}
\end{wraptable}
Next, we conduct an ablation study on the multi-turn search task using Qwen2.5-7B under LLM non-sharing. We compare four normalization configurations: global statistics $(\mu,\sigma)$ (i.e., GRPO), per-agent mean with global standard deviation $(\mu_k,\sigma)$, global mean with per-agent standard deviation $(\mu,\sigma_k)$, and fully per-agent normalization $(\mu_k,\sigma_k)$ (i.e., \methodname{}).

As shown in Table~\ref{tab:ablation}, GRPO performs poorly, indicating that global normalization is a poor match for MAS training where agents play different roles and thus exhibit different advantage distributions. Adding \emph{either} per-agent mean or per-agent standard deviation already brings large improvements. The per-agent standard deviation $(\mu,\sigma_k)$ brings a bigger gain, likely because agents differ more in the spread of their advantages than in the average level. Finally, combining both per-agent mean and standard deviation $(\mu_k,\sigma_k)$, \methodname{} achieves the best results, showing that setting both the mean and the scale per agent gives the most reliable learning signal.

\subsection{Heterogeneous Model Assignment}
\label{sec:hetero_search_cost}
\begin{wrapfigure}{r}{0.55\textwidth}
\vspace{-0.1in}
   \begin{center}
   \includegraphics[width=1.0\linewidth]{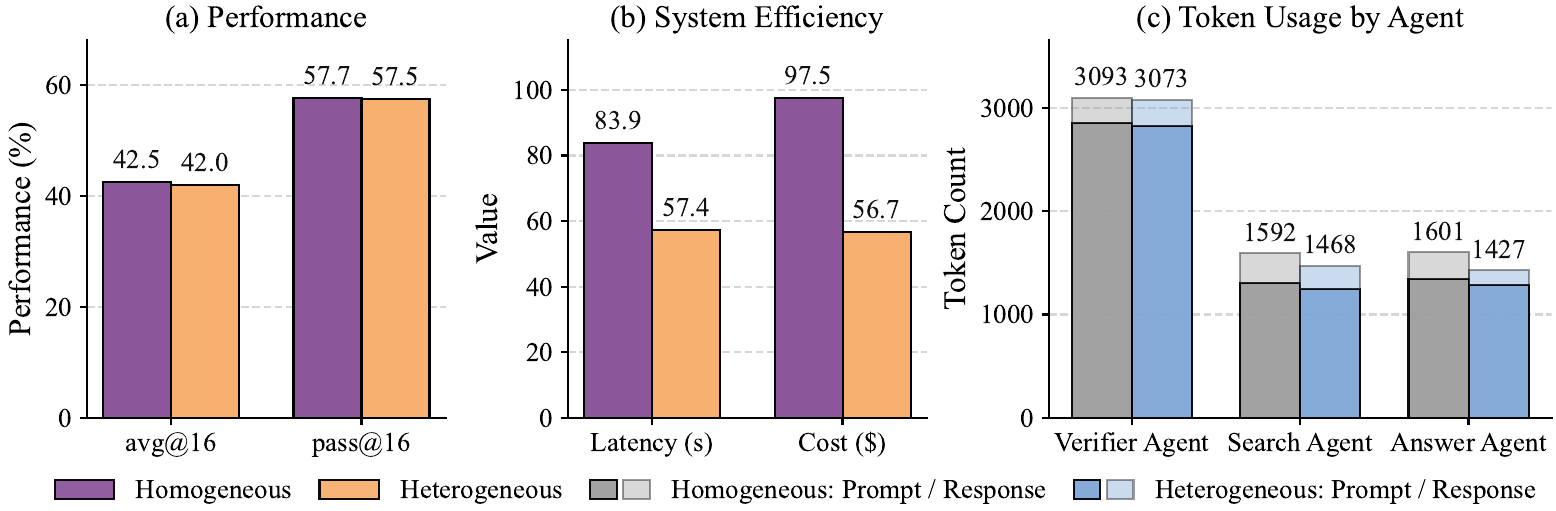}
   \end{center}
   \vspace{-0.1in}
    \caption{Performance and efficiency comparison between homogeneous (all 7B models) and heterogeneous (7B for Verifier, 3B for Search/Answer) model assignment on search tasks.
    Token counts are the average tokens per trajectory for each agent.
    Cost (\$) is estimated using OpenRouter market prices (7B: \$0.30/M tokens, 3B: \$0.06/M tokens) and reported as the total inference cost over the full test set (51.7k samples).}
   \label{fig:heterogeneous_efficiency}
   \vspace{-0.15in}
\end{wrapfigure}
At last, we explore the practical efficiency of \methodname{} when applied to heterogeneous agent-model assignments, where agents with different capacities are combined to optimize performance and cost. We compare a homogeneous baseline where all three agents (verifier, search, and answer) use Qwen2.5-7B, against a heterogeneous setting where the verifier uses Qwen2.5-7B while the search and answer agents use Llama-3.2-3B-Instruct~\citep{grattafiori2024llama}.

As shown in Figure~\ref{fig:heterogeneous_efficiency}, the heterogeneous system maintains performance levels nearly identical to the all-7B baseline, and the average token usage per trajectory remains comparable, with the heterogeneous setup even showing a slight reduction in total volume. This suggests that, in hierarchical multi-agent system, assigning a stronger model to the top-level verifier is sufficient to preserve overall decision quality. By deploying smaller, more efficient models to the low-level agents, the heterogeneous system achieves a 31.6\% reduction in latency and a 41.8\% reduction in total API cost. These findings demonstrate that strategic agent-model assignment facilitates a more flexible and cost-effective multi-agent deployment without sacrificing performance.


\section{Conclusions and Limitations}\label{sec:conclusion}
We studied RL post-training for multi-agent LLM systems and showed that directly extending GRPO with a single global advantage baseline can be brittle under heterogeneous agent reward statistics, leading to gradient spikes and unstable training. To address this, we proposed \methodname{}, which normalizes advantages using each agent's own reward mean and variance. We also built an end-to-end framework supporting multi-agent orchestration, LLM sharing/non-sharing, per-agent optimization, and efficient resource pooling. Across two-agent math reasoning and three-agent multi-turn search, \methodname{} consistently improves over vanilla GRPO and yields more stable training under both sharing and non-sharing settings. Despite these improvements, \methodname{} does not resolve all sources of instability in multi-agent LLM RL (e.g., credit assignment across agents and turns). In addition, our evaluation covers two representative workflows and two open-source model families; its generality to broader agent architectures, tasks, and larger-scale models remains to be further studied.

\bibliography{neurips_2026}

@book{sutton2018reinforcement,
  title={Reinforcement Learning: An Introduction},
  author={Sutton, Richard S and Barto, Andrew G},
  year={2018},
  publisher={MIT press}
}

@article{schulman2017proximal,
  title={Proximal policy optimization algorithms},
  author={Schulman, John and Wolski, Filip and Dhariwal, Prafulla and Radford, Alec and Klimov, Oleg},
  journal={arXiv preprint arXiv:1707.06347},
  year={2017}
}

@inproceedings{
tan2024cradle,
title={Cradle: Empowering foundation agents towards general computer control},
author={Tan, Weihao and Zhang, Wentao and Xu, Xinrun and Xia, Haochong and Ding, Ziluo and Li, Boyu and Zhou, Bohan and Yue, Junpeng and Jiang, Jiechuan and Li, Yewen and others},
booktitle={International Conference on Machine Learning},
year={2025},
}

@article{achiam2023gpt,
  title={{GPT-4} technical report},
  author={Achiam, Josh and Adler, Steven and Agarwal, Sandhini and Ahmad, Lama and Akkaya, Ilge and Aleman, Florencia Leoni and Almeida, Diogo and Altenschmidt, Janko and Altman, Sam and Anadkat, Shyamal and others},
  journal={arXiv preprint arXiv:2303.08774},
  year={2023}
}

@inproceedings{wang2024mobile,
  title={{Mobile-Agent-v2}: Mobile Device Operation Assistant with Effective Navigation via Multi-Agent Collaboration},
  author={Wang, Junyang and Xu, Haiyang and Jia, Haitao and Zhang, Xi and Yan, Ming and Shen, Weizhou and Zhang, Ji and Huang, Fei and Sang, Jitao},
  booktitle={Advances in Neural Information Processing Systems},
  volume={37},
  pages={2686--2710},
  year={2024}
}

@article{team2023gemini,
  title={Gemini: A family of highly capable multimodal models},
  author={Team, Gemini and Anil, Rohan and Borgeaud, Sebastian and Alayrac, Jean-Baptiste and Yu, Jiahui and Soricut, Radu and Schalkwyk, Johan and Dai, Andrew M and Hauth, Anja and Millican, Katie and others},
  journal={arXiv preprint arXiv:2312.11805},
  year={2023}
}

@inproceedings{bai2024digirl,
  title={{DigiRL}: Training In-The-Wild Device-Control Agents with Autonomous Reinforcement Learning},
  author={Bai, Hao and Zhou, Yifei and Pan, Jiayi and Cemri, Mert and Suhr, Alane and Levine, Sergey and Kumar, Aviral},
  booktitle={Advances in Neural Information Processing Systems},
  volume={37},
  pages={12461--12495},
  year={2024}
}

@inproceedings{ouyang2022training,
  title={Training language models to follow instructions with human feedback},
  author={Ouyang, Long and Wu, Jeffrey and Jiang, Xu and Almeida, Diogo and Wainwright, Carroll and Mishkin, Pamela and Zhang, Chong and Agarwal, Sandhini and Slama, Katarina and Ray, Alex and others},
  booktitle={Advances in Neural Information Processing Systems},
  volume={35},
  pages={27730--27744},
  year={2022}
}

@article{liu2024deepseek,
  title={{DeepSeek-V3} Technical Report},
  author={Liu, Aixin and Feng, Bei and Xue, Bing and Wang, Bingxuan and Wu, Bochao and Lu, Chengda and Zhao, Chenggang and Deng, Chengqi and Zhang, Chenyu and Ruan, Chong and others},
  journal={arXiv preprint arXiv:2412.19437},
  year={2024}
}

@inproceedings{rafailov2024direct,
  title={Direct preference optimization: Your language model is secretly a reward model},
  author={Rafailov, Rafael and Sharma, Archit and Mitchell, Eric and Manning, Christopher D and Ermon, Stefano and Finn, Chelsea},
  booktitle={Advances in Neural Information Processing Systems},
  volume={36},
  year={2024}
}

@article{yu2025dapo,
  title={{DAPO}: An open-source {LLM} reinforcement learning system at scale},
  author={Yu, Qiying and Zhang, Zheng and Zhu, Ruofei and Yuan, Yufeng and Zuo, Xiaochen and Yue, Yu and Fan, Tiantian and Liu, Gaohong and Liu, Lingjun and Liu, Xin and others},
  journal={arXiv preprint arXiv:2503.14476},
  year={2025}
}

@article{yang2024qwen2,
  title={Qwen2. 5 technical report},
  author={Yang, An and Yang, Baosong and Zhang, Beichen and Hui, Binyuan and Zheng, Bo and Yu, Bowen and Li, Chengyuan and Liu, Dayiheng and Huang, Fei and Wei, Haoran and others},
  journal={arXiv preprint arXiv:2412.15115},
  year={2024}
}

@inproceedings{ahmadian2024back,
  title={Back to Basics: Revisiting REINFORCE Style Optimization for Learning from Human Feedback in {LLMs}},
  author={Ahmadian, Arash and Cremer, Chris and Gall{\'e}, Matthias and Fadaee, Marzieh and Kreutzer, Julia and Pietquin, Olivier and {\"U}st{\"u}n, Ahmet and Hooker, Sara},
  booktitle={Proceedings of the 62nd Annual Meeting of the Association for Computational Linguistics (Volume 1: Long Papers)},
  year={2024},
  pages={12248--12267},
}

@inproceedings{kool2019buy,
  title={Buy 4 reinforce samples, get a baseline for free!},
  author={Kool, Wouter and van Hoof, Herke and Welling, Max},
  booktitle={ICLR 2019 Workshop},
  year={2019},
}

@article{liu2025understanding,
  title={Understanding {R1}-Zero-Like Training: A Critical Perspective},
  author={Liu, Zichen and Chen, Changyu and Li, Wenjun and Qi, Penghui and Pang, Tianyu and Du, Chao and Lee, Wee Sun and Lin, Min},
  journal={arXiv preprint arXiv:2503.20783},
  year={2025}
}

@article{shao2024deepseekmath,
  title={{DeepSeekMath}: Pushing the Limits of Mathematical Reasoning in Open Language Models},
  author={Shao, Zhihong and Wang, Peiyi and Zhu, Qihao and Xu, Runxin and Song, Junxiao and Bi, Xiao and Zhang, Haowei and Zhang, Mingchuan and Li, YK and Wu, Y and others},
  journal={arXiv preprint arXiv:2402.03300},
  year={2024}
}

@article{jin2025search,
  title={{Search-R1}: Training {LLMs} to Reason and Leverage Search Engines with Reinforcement Learning},
  author={Jin, Bowen and Zeng, Hansi and Yue, Zhenrui and Wang, Dong and Zamani, Hamed and Han, Jiawei},
  journal={arXiv preprint arXiv:2503.09516},
  year={2025}
}

@article{ziegler2019fine,
  title={Fine-tuning language models from human preferences},
  author={Ziegler, Daniel M and Stiennon, Nisan and Wu, Jeffrey and Brown, Tom B and Radford, Alec and Amodei, Dario and Christiano, Paul and Irving, Geoffrey},
  journal={arXiv preprint arXiv:1909.08593},
  year={2019}
}

@inproceedings{stiennon2020learning,
  title={Learning to summarize with human feedback},
  author={Stiennon, Nisan and Ouyang, Long and Wu, Jeffrey and Ziegler, Daniel and Lowe, Ryan and Voss, Chelsea and Radford, Alec and Amodei, Dario and Christiano, Paul F},
  booktitle={Advances in Neural Information Processing Systems},
  volume={33},
  pages={3008--3021},
  year={2020}
}

@inproceedings{zhou2024archer,
  title={{ArCHer}: Training Language Model Agents via Hierarchical Multi-Turn {RL}},
  author={Zhou, Yifei and Zanette, Andrea and Pan, Jiayi and Levine, Sergey and Kumar, Aviral},
  booktitle={International Conference on Machine Learning},
  pages={62178--62209},
  year={2024},
  organization={PMLR}
}

@article{qian2025toolrl,
  title={{ToolRL}: Reward is All Tool Learning Needs},
  author={Qian, Cheng and Acikgoz, Emre Can and He, Qi and Wang, Hongru and Chen, Xiusi and Hakkani-T{\"u}r, Dilek and Tur, Gokhan and Ji, Heng},
  journal={arXiv preprint arXiv:2504.13958},
  year={2025}
}

@article{sheng2024hybridflow,
  title={{HybridFlow}: A Flexible and Efficient {RLHF} Framework},
  author={Sheng, Guangming and Zhang, Chi and Ye, Zilingfeng and Wu, Xibin and Zhang, Wang and Zhang, Ru and Peng, Yanghua and Lin, Haibin and Wu, Chuan},
  journal={arXiv preprint arXiv:2409.19256},
  year={2024}
}

@article{bai2025qwen2,
  title={Qwen2.5-{VL} Technical Report},
  author={Bai, Shuai and Chen, Keqin and Liu, Xuejing and Wang, Jialin and Ge, Wenbin and Song, Sibo and Dang, Kai and Wang, Peng and Wang, Shijie and Tang, Jun and others},
  journal={arXiv preprint arXiv:2502.13923},
  year={2025}
}

@article{wang2025ragen,
  title={{RAGEN}: Understanding Self-Evolution in {LLM} Agents via Multi-Turn Reinforcement Learning},
  author={Wang, Zihan and Wang, Kangrui and Wang, Qineng and Zhang, Pingyue and Li, Linjie and Yang, Zhengyuan and Yu, Kefan and Nguyen, Minh Nhat and Liu, Licheng and Gottlieb, Eli and others},
  journal={arXiv preprint arXiv:2504.20073},
  year={2025}
}

@article{kwiatkowski2019natural,
  title={Natural questions: a benchmark for question answering research},
  author={Kwiatkowski, Tom and Palomaki, Jennimaria and Redfield, Olivia and Collins, Michael and Parikh, Ankur and Alberti, Chris and Epstein, Danielle and Polosukhin, Illia and Devlin, Jacob and Lee, Kenton and others},
  journal={Transactions of the Association for Computational Linguistics},
  volume={7},
  pages={453--466},
  year={2019},
  publisher={MIT Press One Rogers Street, Cambridge, MA 02142-1209, USA journals-info~…}
}

@article{yang2018hotpotqa,
  title={{HotpotQA}: A dataset for diverse, explainable multi-hop question answering},
  author={Yang, Zhilin and Qi, Peng and Zhang, Saizheng and Bengio, Yoshua and Cohen, William W and Salakhutdinov, Ruslan and Manning, Christopher D},
  journal={arXiv preprint arXiv:1809.09600},
  year={2018}
}

@article{trivedi2022musique,
  title={{MuSiQue}: Multihop Questions via Single-hop Question Composition},
  author={Trivedi, Harsh and Balasubramanian, Niranjan and Khot, Tushar and Sabharwal, Ashish},
  journal={Transactions of the Association for Computational Linguistics},
  volume={10},
  pages={539--554},
  year={2022},
  publisher={MIT Press One Broadway, 12th Floor, Cambridge, Massachusetts 02142, USA~…}
}

@article{press2022measuring,
  title={Measuring and narrowing the compositionality gap in language models},
  author={Press, Ofir and Zhang, Muru and Min, Sewon and Schmidt, Ludwig and Smith, Noah A and Lewis, Mike},
  journal={arXiv preprint arXiv:2210.03350},
  year={2022}
}

@article{mallen2022not,
  title={When not to trust language models: Investigating effectiveness of parametric and non-parametric memories},
  author={Mallen, Alex and Asai, Akari and Zhong, Victor and Das, Rajarshi and Khashabi, Daniel and Hajishirzi, Hannaneh},
  journal={arXiv preprint arXiv:2212.10511},
  year={2022}
}

@article{ho2020constructing,
  title={Constructing a multi-hop {QA} dataset for comprehensive evaluation of reasoning steps},
  author={Ho, Xanh and Nguyen, Anh-Khoa Duong and Sugawara, Saku and Aizawa, Akiko},
  journal={arXiv preprint arXiv:2011.01060},
  year={2020}
}

@article{joshi2017triviaqa,
  title={{TriviaQA}: A large scale distantly supervised challenge dataset for reading comprehension},
  author={Joshi, Mandar and Choi, Eunsol and Weld, Daniel S and Zettlemoyer, Luke},
  journal={arXiv preprint arXiv:1705.03551},
  year={2017}
}

@article{wang2022text,
  title={Text embeddings by weakly-supervised contrastive pre-training},
  author={Wang, Liang and Yang, Nan and Huang, Xiaolong and Jiao, Binxing and Yang, Linjun and Jiang, Daxin and Majumder, Rangan and Wei, Furu},
  journal={arXiv preprint arXiv:2212.03533},
  year={2022}
}

@inproceedings{
feng2025towards,
title={Towards Efficient Online Tuning of {VLM} Agents via Counterfactual Soft Reinforcement Learning},
author={Lang Feng and Weihao Tan and Zhiyi Lyu and Longtao Zheng and Haiyang Xu and Ming Yan and Fei Huang and Bo An},
booktitle={International Conference on Machine Learning},
volume={267},
pages={16884--16903},
year={2025},
}

@article{yang2025qwen3,
  title={Qwen3 technical report},
  author={Yang, An and Li, Anfeng and Yang, Baosong and Zhang, Beichen and Hui, Binyuan and Zheng, Bo and Yu, Bowen and Gao, Chang and Huang, Chengen and Lv, Chenxu and others},
  journal={arXiv preprint arXiv:2505.09388},
  year={2025}
}

@article{zheng2025group,
  title={Group sequence policy optimization},
  author={Zheng, Chujie and Liu, Shixuan and Li, Mingze and Chen, Xiong-Hui and Yu, Bowen and Gao, Chang and Dang, Kai and Liu, Yuqiong and Men, Rui and Yang, An and others},
  journal={arXiv preprint arXiv:2507.18071},
  year={2025}
}

@article{zhang2025agentorchestra,
  title={{AgentOrchestra}: Orchestrating Hierarchical Multi-Agent Intelligence with the Tool-Environment-Agent ({TEA}) Protocol},
  author={Zhang, Wentao and Zeng, Liang and Xiao, Yuzhen and Li, Yongcong and Cui, Ce and Zhao, Yilei and Hu, Rui and Liu, Yang and Zhou, Yahui and An, Bo},
  journal={arXiv preprint arXiv:2506.12508},
  year={2025}
}

@inproceedings{chang2025main,
  title={{MAIN-RAG}: Multi-agent filtering retrieval-augmented generation},
  author={Chang, Chia-Yuan and Jiang, Zhimeng and Rakesh, Vineeth and Pan, Menghai and Yeh, Chin-Chia Michael and Wang, Guanchu and Hu, Mingzhi and Xu, Zhichao and Zheng, Yan and Das, Mahashweta and others},
  booktitle={Proceedings of the 63rd Annual Meeting of the Association for Computational Linguistics (Volume 1: Long Papers)},
  pages={2607--2622},
  year={2025}
}

@inproceedings{qian2024chatdev,
  title={{ChatDev}: Communicative agents for software development},
  author={Qian, Chen and Liu, Wei and Liu, Hongzhang and Chen, Nuo and Dang, Yufan and Li, Jiahao and Yang, Cheng and Chen, Weize and Su, Yusheng and Cong, Xin and others},
  booktitle={Proceedings of the 62nd Annual Meeting of the Association for Computational Linguistics (Volume 1: Long Papers)},
  pages={15174--15186},
  year={2024}
}

@inproceedings{
hong2024metagpt,
title={Meta{GPT}: Meta Programming for A Multi-Agent Collaborative Framework},
author={Sirui Hong and Mingchen Zhuge and Jonathan Chen and Xiawu Zheng and Yuheng Cheng and Jinlin Wang and Ceyao Zhang and Zili Wang and Steven Ka Shing Yau and Zijuan Lin and Liyang Zhou and Chenyu Ran and Lingfeng Xiao and Chenglin Wu and J{\"u}rgen Schmidhuber},
booktitle={The Twelfth International Conference on Learning Representations},
year={2024},
}

@inproceedings{feng2025group,
  title={Group-in-Group Policy Optimization for {LLM} Agent Training},
  author={Feng, Lang and Xue, Zhenghai and Liu, Tingcong and An, Bo},
  booktitle={Advances in Neural Information Processing Systems},
  year={2025}
}

@article{luo2025agent,
  title={Agent lightning: Train any {AI} agents with reinforcement learning},
  author={Luo, Xufang and Zhang, Yuge and He, Zhiyuan and Wang, Zilong and Zhao, Siyun and Li, Dongsheng and Qiu, Luna K and Yang, Yuqing},
  journal={arXiv preprint arXiv:2508.03680},
  year={2025}
}

@article{wang2025reinforcement,
  title={Reinforcement Learning Optimization for Large-Scale Learning: An Efficient and User-Friendly Scaling Library},
  author={Wang, Weixun and Xiong, Shaopan and Chen, Gengru and Gao, Wei and Guo, Sheng and He, Yancheng and Huang, Ju and Liu, Jiaheng and Li, Zhendong and Li, Xiaoyang and others},
  journal={arXiv preprint arXiv:2506.06122},
  year={2025}
}

@article{fu2025areal,
  title={{AReaL}: A Large-Scale Asynchronous Reinforcement Learning System for Language Reasoning},
  author={Fu, Wei and Gao, Jiaxuan and Shen, Xujie and Zhu, Chen and Mei, Zhiyu and He, Chuyi and Xu, Shusheng and Wei, Guo and Mei, Jun and Wang, Jiashu and others},
  journal={arXiv preprint arXiv:2505.24298},
  year={2025}
}

@article{jiang2025verltool,
  title={{VerlTool}: Towards holistic agentic reinforcement learning with tool use},
  author={Jiang, Dongfu and Lu, Yi and Li, Zhuofeng and Lyu, Zhiheng and Nie, Ping and Wang, Haozhe and Su, Alex and Chen, Hui and Zou, Kai and Du, Chao and others},
  journal={arXiv preprint arXiv:2509.01055},
  year={2025}
}

@article{dong2025agentic,
  title={Agentic reinforced policy optimization},
  author={Dong, Guanting and Mao, Hangyu and Ma, Kai and Bao, Licheng and Chen, Yifei and Wang, Zhongyuan and Chen, Zhongxia and Du, Jiazhen and Wang, Huiyang and Zhang, Fuzheng and others},
  journal={arXiv preprint arXiv:2507.19849},
  year={2025}
}

@article{hendrycks2021measuring,
  title={Measuring mathematical problem solving with the math dataset},
  author={Hendrycks, Dan and Burns, Collin and Kadavath, Saurav and Arora, Akul and Basart, Steven and Tang, Eric and Song, Dawn and Steinhardt, Jacob},
  journal={arXiv preprint arXiv:2103.03874},
  year={2021}
}

@article{zhao2025stronger,
  title={Stronger Together: On-Policy Reinforcement Learning for Collaborative {LLMs}},
  author={Zhao, Yujie and Hu, Lanxiang and Wang, Yang and Hou, Minmin and Zhang, Hao and Ding, Ke and Zhao, Jishen},
  journal={arXiv preprint arXiv:2510.11062},
  year={2025}
}

@article{hu2024openrlhf,
  title={{OpenRLHF}: An Easy-to-use, Scalable and High-performance {RLHF} Framework},
  author={Jian Hu and Xibin Wu and Zilin Zhu and Xianyu and Weixun Wang and Dehao Zhang and Yu Cao},
  journal={arXiv preprint arXiv:2405.11143},
  year={2024}
}

@article{chen2025heterogeneous,
  title={Heterogeneous Group-Based Reinforcement Learning for {LLM}-based Multi-Agent Systems},
  author={Chen, Guanzhong and Yang, Shaoxiong and Li, Chao and Liu, Wei and Luan, Jian and Xu, Zenglin},
  journal={arXiv preprint arXiv:2506.02718},
  year={2025}
}

@article{wan2025rema,
  title={{ReMA}: Learning to meta-think for {LLMs} with multi-agent reinforcement learning},
  author={Wan, Ziyu and Li, Yunxiang and Wen, Xiaoyu and Song, Yan and Wang, Hanjing and Yang, Linyi and Schmidt, Mark and Wang, Jun and Zhang, Weinan and Hu, Shuyue and others},
  journal={arXiv preprint arXiv:2503.09501},
  year={2025}
}

@article{li2025flow,
  title={In-the-Flow Agentic System Optimization for Effective Planning and Tool Use},
  author={Li, Zhuofeng and Zhang, Haoxiang and Han, Seungju and Liu, Sheng and Xie, Jianwen and Zhang, Yu and Choi, Yejin and Zou, James and Lu, Pan},
  journal={arXiv preprint arXiv:2510.05592},
  year={2025}
}

@article{li2025chain,
  title={Chain-of-agents: End-to-end agent foundation models via multi-agent distillation and agentic {RL}},
  author={Li, Weizhen and Lin, Jianbo and Jiang, Zhuosong and Cao, Jingyi and Liu, Xinpeng and Zhang, Jiayu and Huang, Zhenqiang and Chen, Qianben and Sun, Weichen and Wang, Qiexiang and others},
  journal={arXiv preprint arXiv:2508.13167},
  year={2025}
}

@article{xue2025comas,
  title={{CoMAS}: Co-Evolving Multi-Agent Systems via Interaction Rewards},
  author={Xue, Xiangyuan and Zhou, Yifan and Zhang, Guibin and Zhang, Zaibin and Li, Yijiang and Zhang, Chen and Yin, Zhenfei and Torr, Philip and Ouyang, Wanli and Bai, Lei},
  journal={arXiv preprint arXiv:2510.08529},
  year={2025}
}

@misc{marti2025,
  title={{MARTI}: A Framework for Multi-Agent {LLM} Systems Reinforced Training and Inference},
  author={Kaiyan Zhang and Runze Liu and Xuekai Zhu and Kai Tian and Sihang Zeng and Guoli Jia and Yuchen Fan and Xingtai Lv and Yuxin Zuo and Che Jiang and Ziyang Liu and Jianyu Wang and Yuru Wang and Ruotong Zhao and Ermo Hua and Yibo Wang and Shijie Wang and Junqi Gao and Xinwei Long and Youbang Sun and Zhiyuan Ma and Ganqu Cui and Lei Bai and Ning Ding and Biqing Qi and Bowen Zhou},
  year={2025},
  institution={Tsinghua University and Shanghai AI Lab},
  url={https://github.com/TsinghuaC3I/MARTI}
}

@inproceedings{park2025maporl,
  title={{MAPoRL}: Multi-agent post-co-training for collaborative large language models with reinforcement learning},
  author={Park, Chanwoo and Han, Seungju and Guo, Xingzhi and Ozdaglar, Asuman E and Zhang, Kaiqing and Kim, Joo-Kyung},
  booktitle={Proceedings of the 63rd Annual Meeting of the Association for Computational Linguistics (Volume 1: Long Papers)},
  pages={30215--30248},
  year={2025}
}

@article{liu2025spiral,
  title={{SPIRAL}: Self-Play on Zero-Sum Games Incentivizes Reasoning via Multi-Agent Multi-Turn Reinforcement Learning},
  author={Liu, Bo and Guertler, Leon and Yu, Simon and Liu, Zichen and Qi, Penghui and Balcells, Daniel and Liu, Mickel and Tan, Cheston and Shi, Weiyan and Lin, Min and others},
  journal={arXiv preprint arXiv:2506.24119},
  year={2025}
}

@article{liu2025llm,
  title={{LLM} collaboration with multi-agent reinforcement learning},
  author={Liu, Shuo and Chen, Tianle and Liang, Zeyu and Lyu, Xueguang and Amato, Christopher},
  journal={arXiv preprint arXiv:2508.04652},
  year={2025}
}

@article{xue2025simpletir,
  title={{SimpleTIR}: End-to-end reinforcement learning for multi-turn tool-integrated reasoning},
  author={Xue, Zhenghai and Zheng, Longtao and Liu, Qian and Li, Yingru and Zheng, Xiaosen and Ma, Zejun and An, Bo},
  journal={arXiv preprint arXiv:2509.02479},
  year={2025}
}

@article{zeng2025simplerl,
  title={{SimpleRL-Zoo}: Investigating and taming zero reinforcement learning for open base models in the wild},
  author={Zeng, Weihao and Huang, Yuzhen and Liu, Qian and Liu, Wei and He, Keqing and Ma, Zejun and He, Junxian},
  journal={arXiv preprint arXiv:2503.18892},
  year={2025}
}

@article{wang2025spa,
  title={{SPA-RL}: Reinforcing {LLM} Agents via Stepwise Progress Attribution},
  author={Wang, Hanlin and Leong, Chak Tou and Wang, Jiashuo and Wang, Jian and Li, Wenjie},
  journal={arXiv preprint arXiv:2505.20732},
  year={2025}
}

@article{lai2025computerrl,
  title={{ComputerRL}: Scaling end-to-end online reinforcement learning for computer use agents},
  author={Lai, Hanyu and Liu, Xiao and Zhao, Yanxiao and Xu, Han and Zhang, Hanchen and Jing, Bohao and Ren, Yanyu and Yao, Shuntian and Dong, Yuxiao and Tang, Jie},
  journal={arXiv preprint arXiv:2508.14040},
  year={2025}
}

@article{zhang2025agent,
  title={Agent Learning via Early Experience},
  author={Zhang, Kai and Chen, Xiangchao and Liu, Bo and Xue, Tianci and Liao, Zeyi and Liu, Zhihan and Wang, Xiyao and Ning, Yuting and Chen, Zhaorun and Fu, Xiaohan and others},
  journal={arXiv preprint arXiv:2510.08558},
  year={2025}
}

@misc{slime_github,
  author = {Zilin Zhu and Chengxing Xie and Xin Lv and slime Contributors},
  title = {slime: An {LLM} post-training framework for {RL} Scaling},
  year = {2025},
  howpublished = {\url{https://github.com/THUDM/slime}},
  urldate = {2025-06-19}
}

@article{feng2026agentocr,
  title={{AgentOCR}: Reimagining Agent History via Optical Self-Compression},
  author={Feng, Lang and Yang, Fuchao and Chen, Feng and Cheng, Xin and Xu, Haiyang and Wan, Zhenglin and Yan, Ming and An, Bo},
  journal={arXiv preprint arXiv:2601.04786},
  year={2026}
}

@inproceedings{he2024olympiadbench,
  title={{OlympiadBench}: A challenging benchmark for promoting {AGI} with olympiad-level bilingual multimodal scientific problems},
  author={He, Chaoqun and Luo, Renjie and Bai, Yuzhuo and Hu, Shengding and Thai, Zhen and Shen, Junhao and Hu, Jinyi and Han, Xu and Huang, Yujie and Zhang, Yuxiang and others},
  booktitle={Proceedings of the 62nd Annual Meeting of the Association for Computational Linguistics (Volume 1: Long Papers)},
  pages={3828--3850},
  year={2024}
}

@article{tran2025multi,
  title={Multi-agent collaboration mechanisms: A survey of {LLMs}},
  author={Tran, Khanh-Tung and Dao, Dung and Nguyen, Minh-Duong and Pham, Quoc-Viet and O'Sullivan, Barry and Nguyen, Hoang D},
  journal={arXiv preprint arXiv:2501.06322},
  year={2025}
}

@article{grattafiori2024llama,
  title={The {Llama} 3 herd of models},
  author={Grattafiori, Aaron and Dubey, Abhimanyu and Jauhri, Abhinav and Pandey, Abhinav and Kadian, Abhishek and Al-Dahle, Ahmad and Letman, Aiesha and Mathur, Akhil and Schelten, Alan and Vaughan, Alex and others},
  journal={arXiv preprint arXiv:2407.21783},
  year={2024}
}

@inproceedings{zheng2024sglang,
  title={{SGL}ang: Efficient Execution of Structured Language Model Programs},
  author={Zheng, Lianmin and Yin, Liangsheng and Xie, Zhiqiang and Sun, Chuyue Livia and Huang, Jeff and Yu, Cody Hao and Cao, Shiyi and Kozyrakis, Christos and Stoica, Ion and Gonzalez, Joseph E and others},
  booktitle={Advances in Neural Information Processing Systems},
  volume={37},
  pages={62557--62583},
  year={2024}
}

@article{zhang2025landscape,
  title={The landscape of agentic reinforcement learning for {LLMs}: A survey},
  author={Zhang, Guibin and Geng, Hejia and Yu, Xiaohang and Yin, Zhenfei and Zhang, Zaibin and Tan, Zelin and Zhou, Heng and Li, Zhongzhi and Xue, Xiangyuan and Li, Yijiang and others},
  journal={arXiv preprint arXiv:2509.02547},
  year={2025}
}

@article{yuan2025marshal,
  title={{MARSHAL}: Incentivizing Multi-Agent Reasoning via Self-Play with Strategic {LLMs}},
  author={Yuan, Huining and Xu, Zelai and Tan, Zheyue and Yi, Xiangmin and Guang, Mo and Long, Kaiwen and Hui, Haojia and Li, Boxun and Chen, Xinlei and Zhao, Bo and others},
  journal={arXiv preprint arXiv:2510.15414},
  year={2025}
}

@article{hong2025multi,
  title={Multi-Agent Deep Research: Training Multi-Agent Systems with {M-GRPO}},
  author={Hong, Haoyang and Yin, Jiajun and Wang, Yuan and Liu, Jingnan and Chen, Zhe and Yu, Ailing and Li, Ji and Ye, Zhiling and Xiao, Hansong and Chen, Yefei and others},
  journal={arXiv preprint arXiv:2511.13288},
  year={2025}
}

@article{mo2025multi,
  title={Multi-Agent Tool-Integrated Policy Optimization},
  author={Mo, Zhanfeng and Li, Xingxuan and Chen, Yuntao and Bing, Lidong},
  journal={arXiv preprint arXiv:2510.04678},
  year={2025}
}
\bibliographystyle{unsrt}

\newpage
\appendix
\section{Broader Impacts}\label{app:broader_impacts}
This work improves the stability and efficiency of RL post-training for cooperative multi-agent LLM systems. Potential positive impacts include more reliable multi-agent reasoning and tool use, lower computational waste from unstable training, and more cost-effective deployment through flexible agent-model assignment. These benefits may support applications such as mathematical reasoning, knowledge-intensive question answering, and other tasks requiring coordinated LLM agents.

However, more stable and efficient multi-agent training may also lower the barrier to building powerful agentic systems, which could be misused for harmful automation, misinformation, or other unsafe workflows. Trained agents may also inherit biases and failure modes from the underlying models, datasets, and reward functions. \methodname{} is not a standalone safety mechanism; real-world deployment should include appropriate safety evaluation, monitoring, access control, human oversight, and compliance with relevant data and model licenses.

\section{Proofs}
\label{app:proofs}  
\subsection{Proof of Lemma~\ref{lem:global-second-moment}}
\label{app:proof_lemma_global_second_moment}

\textbf{Lemma~\ref{lem:global-second-moment}}.
\textit{
Under Assumptions~\ref{assump:score}, for any agent $k$,
\begin{equation}
\nonumber
\mathbb{E}_{\bm{a}_t^i\sim \mathcal{Y}_k}\big[\|\tilde{g}_k^{\mathrm{global}}\|^2\big]
{=} \mathbb{E}_{\bm{a}_t^i\sim \mathcal{Y}_k}\big[\|z_{i,t}^{(k)}\|^2\big]\,
  \frac{\sigma_k^2 + (\mu_k - \mu)^2}{\sigma^2} + \Delta_k,
\end{equation}
where $\mu_k \triangleq \frac{1}{|\mathcal{Y}_k|}\sum_{\bm{a}_t^i\in \mathcal{Y}_k} R^i$, $\sigma_k^2 \triangleq \frac{1}{|\mathcal{Y}_k|}\sum_{\bm{a}_t^i\in \mathcal{Y}_k} (R^i - \mu_k)^2$
are the mean and variance when sampling
time steps uniformly from $\mathcal{Y}_k$ (i.e., when agent $k$ is active). $\Delta_k$ is a score-reward covariance correction term.
}

\begin{proof}
By definition,
\[
    \tilde{g}_k^{\mathrm{global}}
    = \frac{R^i - \mu}{\sigma}\, z_{i,t}^{(k)},
\]
so
\begin{equation}
    \mathbb{E}_{\bm{a}_t^i\sim \mathcal{Y}_k}\big[\|\tilde{g}_k^{\mathrm{global}}\|^2\big]
    = \mathbb{E}_{\bm{a}_t^i\sim \mathcal{Y}_k}\left[
        \|z_{i,t}^{(k)}\|^2 \frac{(R^i-\mu)^2}{\sigma^2}
      \right].
\end{equation}
We can factor the expectation:
\begin{align}
    \mathbb{E}_{\bm{a}_t^i\sim \mathcal{Y}_k}\big[\|\tilde{g}_k^{\mathrm{global}}\|^2\big]
    &= \mathbb{E}_{\bm{a}_t^i\sim \mathcal{Y}_k}\big[\|z_{i,t}^{(k)}\|^2\big]\,
       \mathbb{E}_{\bm{a}_t^i\sim \mathcal{Y}_k}\left[\frac{(R^i-\mu)^2}{\sigma^2}\right] + \mathrm{Cov}(\|z_{i,t}^{(k)}\|^2,\frac{(R^i-\mu)^2}{\sigma^2}).
\end{align}
Next, we use the standard variance decomposition
\begin{align}
    &\mathbb{E}_{\bm{a}_t^i\sim \mathcal{Y}_k}\big[(R^i - \mu)^2\big] \\
    = &\mathbb{E}_{\bm{a}_t^i\sim \mathcal{Y}_k}\big[(R^i - \mu_k + \mu_k - \mu)^2\big] \\
    = &\mathbb{E}_{\bm{a}_t^i\sim \mathcal{Y}_k}\big[(R^i - \mu_k)^2 + 2(\mu_k - \mu)(R^i - \mu_k) + (\mu_k - \mu)^2\big] \\
    = &\sigma^2_k + (\mu_k-\mu)^2.
\end{align}
Substituting this into the previous expression yields
\begin{align}
    \mathbb{E}_{\bm{a}_t^i\sim \mathcal{Y}_k}\big[\|\tilde{g}_k^{\mathrm{global}}\|^2\big]
    &= \mathbb{E}_{\bm{a}_t^i\sim \mathcal{Y}_k}\big[\|z_{i,t}^{(k)}\|^2\big]\,
       \frac{\sigma^2_k + (\mu_k-\mu)^2}{\sigma^2} + \Delta_k,
\end{align}
with $\Delta_k=\mathrm{Cov}(\|z_{i,t}^{(k)}\|^2,\frac{(R^i-\mu)^2}{\sigma^2})$.
\end{proof}
The second term captures the residual score–reward correlation. In large-scale LLM training, rewards are typically low-dimensional signals of final task quality (e.g., pass/fail for reasoning, correctness for coding), while $z_{i,t}^{(k)}$ depends mainly on the local token-level stochasticity of the policy. Empirically, their covariance is
often much smaller than the main scaling factor $\mathbb{E}_k[\|z_{i,t}^{(k)}\|^2](\sigma_k^2 + (\mu_k - \mu)^2)/\sigma^2$.

\subsection{Proof of Proposition~\ref{prop:gradient-inflation}}
\label{app:proof_prop_gradient_inflation}
\textbf{Proposition~\ref{prop:gradient-inflation}} (Gradient-Norm Inflation).
\textit{
As either the normalized mean deviation
$|\mu_k - \mu| / \sigma$ or the normalized variance ratio
$\sigma_k^2 / \sigma^2$ becomes large, the second moment of
$\tilde{g}_k^{\mathrm{global}}$ grows at least linearly. Consequently, along any training process for which
there exists a sequence of iterations indexed by $m$ such that
$$
\frac{\sigma_{k,m}^2 + (\mu_{k,m} - \mu_m)^2}{\sigma_m^2} \to \infty, \quad  \mathbb{E}\bigl[\|\tilde{g}_m^{\mathrm{global}}\|^2\bigr] \to \infty,
$$
where $\tilde{g}_{m}^{\mathrm{global}} = (\tilde{g}_{1,m}^{\mathrm{global}},\dots,\tilde{g}_{K,m}^{\mathrm{global}})$ stacking all LLM agents' gradients.
}

\begin{proof}
The key message is immediate from Lemma~\ref{lem:global-second-moment}: the global-normalized
gradient for agent $k$ is amplified whenever the reward statistics of the
agent-active subset $\mathcal{Y}_k$ do not match the global reward statistics.
Concretely, Lemma~\ref{lem:global-second-moment} gives
\begin{equation}\label{eq:gn-inflation}
\mathbb{E}_{\bm{a}_t^i\sim \mathcal{Y}_k}\!\big[\|\tilde{g}_k^{\mathrm{global}}\|^2\big]
=
\mathbb{E}_{\bm{a}_t^i\sim \mathcal{Y}_k}\!\big[\|z_{i,t}^{(k)}\|^2\big]\,
\frac{\sigma_k^2 + (\mu_k - \mu)^2}{\sigma^2} \;+\; \Delta_k .
\end{equation}
The multiplier
\(
(\sigma_k^2 + (\mu_k-\mu)^2)/\sigma^2
\)
contains two sources of inflation: (i) a variance mismatch $\sigma_k^2/\sigma^2$, meaning rewards observed when agent $k$ is active have a different spread than the global rewards; and (ii) a mean misalignment $(\mu_k-\mu)^2/\sigma^2$, meaning agent-$k$'s active rewards are shifted relative to the global mean. Either effect increases the second moment of $\tilde{g}_k^{\mathrm{global}}$ proportionally, hence causing larger gradient fluctuations.

For the claimed blow-up statement along training, apply Equation~(\ref{eq:gn-inflation}) at
iteration $m$:
\[
\mathbb{E}\big[\|\tilde{g}_{k,m}^{\mathrm{global}}\|^2\big]
=
\mathbb{E}\big[\|z_{i,t,m}^{(k)}\|^2\big]\,
\frac{\sigma_{k,m}^2 + (\mu_{k,m} - \mu_m)^2}{\sigma_m^2}
+\Delta_{k,m}.
\]
Thus, if there exists a subsequence with
\(
\frac{\sigma_{k,m}^2 + (\mu_{k,m} - \mu_m)^2}{\sigma_m^2}\to\infty,
\)
then the second moment of the global-normalized gradient necessarily diverges
unless $\mathbb{E}[\|z_{i,t,m}^{(k)}\|^2]$ or $\Delta_{k,m}$ cancels this growth.
Finally, since the stacked gradient satisfies
\(
\|\tilde{g}_m^{\mathrm{global}}\|^2=\sum_{j=1}^K \|\tilde{g}_{j,m}^{\mathrm{global}}\|^2,
\)
divergence of any component implies
\(
\mathbb{E}\big[\|\tilde{g}_m^{\mathrm{global}}\|^2\big]\to\infty,
\)
which completes the proof.
\end{proof}

\section{Experimental Details}
\label{app:experimental_details}

\subsection{MAS Workflows}\label{app:workflows}
\textbf{Math Orchestration.}~
For mathematical reasoning tasks, we use a two-agent architecture (a \emph{solver agent} and a \emph{verifier agent}). In each episode, the solver agent proposes candidate solutions, while the verifier agent inspects the solver’s reasoning and decides whether the current solution should be accepted or revised. If the verifier deems the solution unsatisfactory, the system triggers another round of solver refinement. Otherwise, the interaction terminates and the final answer is emitted. We use Qwen3-4B/8B~\citep{yang2025qwen3} as the LLM policy for each agent and evaluate both the shared-LLM and non-shared settings.

\textbf{Search Orchestration.}~
For multi-turn search tool-calling task, we design a hierarchical workflow comprising three agents: a \emph{verifier agent}, a \emph{search agent}, and an \emph{answer agent}, as shown in Figure~\ref{fig:orchestration}. At the top level, the verifier agent determines whether the information currently available is sufficient to answer the query. If not, it delegates downward to the search agent, which is responsible for retrieving additional external evidence. Once the verifier agent judges that the information is adequate, it invokes the answer agent, which synthesizes all retrieved evidence into a final answer. We use Qwen2.5-3B/7B~\citep{bai2025qwen2} as the LLM policy for each agent and evaluate both the shared-LLM and non-shared settings.
\begin{figure*}[h]
   \begin{center}
   \includegraphics[width=1.0\linewidth]{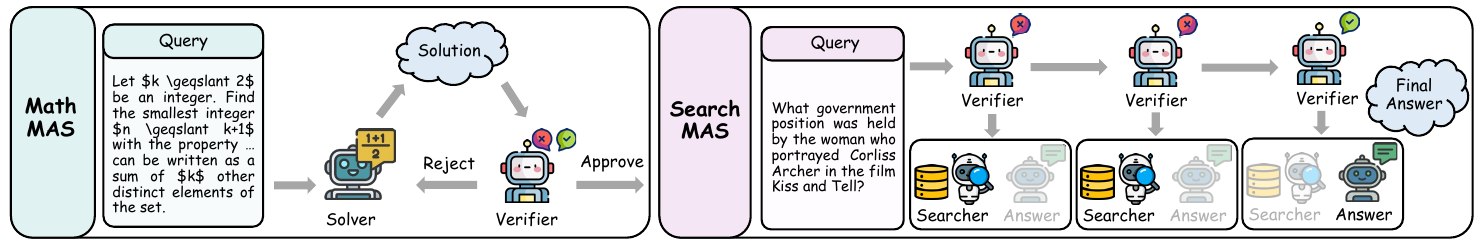}
   \end{center}
   \vspace{-0.1in}
   \caption{Illustration of the orchestrations. \textbf{Left}: Math orchestration uses a two-agent loop, where a solver proposes candidate solutions and a verifier evaluates and either approves or requests refinement. \textbf{Right}: Multi-turn search orchestration uses a hierarchical three-agent pipeline, where a top-level verifier selectively invokes either a search agent to retrieve external information or an answer agent to produce the final result.}
   \label{fig:orchestration}
\end{figure*}

\subsection{Hyperparameters for Math}
\label{app:experimental_details_math}
For the Math task, uniform hyperparameters are employed across all methods. The maximum prompt and response lengths are set to 8192 and 4096 tokens, respectively. We utilize a two-agent orchestration framework (comprising a Solver Agent and a Verifier Agent), allowing for a maximum of two solver-verifier loops. The actor learning rate is fixed at $1 \times 10^{-6}$ for each agent, utilizing on-policy updates. We employ group-based rollouts with a group size of 8. A binary rule-based reward function is used (1 for success, 0 for failure), while invalid actions incur a penalty with a coefficient of 0.1. The batch sizes for training and evaluation are 32 and 64, respectively. During evaluation, we use nucleus sampling with $\text{top\_p}=0.95$ and a temperature of 0.6.

\subsection{Hyperparameters for Multi-Turn Search}
\label{app:experimental_details_search}
Similarly, all methods share identical hyperparameter configurations for the multi-turn search task. The maximum limits for prompts and responses are 4096 and 800 tokens, respectively, with each episode capped at a maximum of 4 turns. This task utilizes a three-agent architecture (Verifier, Search, and Answer). The actor learning rate is set to $1 \times 10^{-6}$ per agent, with 5 update iterations. We use group-based rollouts with a group size of 5. The reward structure employs the same binary rule-based criteria; however, the invalid-action penalty coefficient is set to 0.01. Training and evaluation batch sizes are 128 and 256, respectively. Evaluation sampling parameters remain consistent with the Math task ($\text{top\_p}=0.95$, temperature 0.6).

\subsection{Hardware and System Configuration}
All experiments are conducted on NVIDIA H100 GPUs. 

\subsection{Cost Estimation}
The API costs (\$) are estimated using OpenRouter market prices via the Together provider (\url{https://openrouter.ai/provider/together}). The pricing for Qwen2.5-7B is set at \$0.30 per million (M) tokens for both input and output. For Llama-3.2-3B-Instruct, the pricing is \$0.06 per million (M) tokens for both input and output.

\section{Prompt Templates}

In multi-agent configurations, each agent receives a composite prompt structured as follows:
(i) an environment-provided observation prompt (\texttt{env\_prompt}) describing the task (and interaction history, where applicable);
(ii) the accumulated team interaction context (\texttt{team\_context}); and
(iii) the specific role instruction for the agent.

\subsection{Math Task}
The shared environment prompt for the multi-agent setup is defined as follows:

\begin{promptbox}[title=Environment Prompt]
You are a member of an expert multi-agent team tasked with solving the math problem.
The team's math problem is: \{task\_description\}
\end{promptbox}

The \textbf{Solver Agent} receives the following prompt:

\begin{promptbox}[title=Solver Agent Prompt]
\# Task Introduction

\{env\_prompt\}

\# Your Teammates' Outputs

\{team\_context\}

\# Your Role

You are a "Solver Agent". Your job is to carefully reason through the math problem step by step and derive the correct answer. When reasoning, consider your teammates' outputs (if any) as auxiliary context.

You should give the final answer within \textbackslash boxed\{\}.
\end{promptbox}

The \textbf{Verifier Agent} receives the following prompt:

\begin{promptbox}[title=Verifier Agent Prompt]
\# Task Introduction

\{env\_prompt\}

\# Your Teammates' Outputs

\{team\_context\}

\# Your Role

You are a "Verifier Agent". Your responsibility is to critically review the most recent solution provided by the "Solver Agent". Check each reasoning step, formula, and conclusion for accuracy, completeness, and logical consistency.
At the end of your output, you MUST provide your verdict within <verify> </verify> using exactly one of:

(1) <verify>approve</verify> if all steps and the final answer are correct.

(2) <verify>reject</verify> if you detect any issue.
\end{promptbox}

\subsection{Multi-Turn Search Task (Verifier-Search-Answer)}
The shared environment prompt for the multi-agent setup is defined as follows:

\begin{promptbox}[title=Environment Prompt]
You are a member of an expert multi-agent team tasked with answering the given question step-by-step.
The question is: \{task\_description\}

Your team can access an external search engine to retrieve external information. At each step, you and your teammates must collaborate to make progress toward answering the question.

Prior to this step, your team has already taken \{step\_count\} step(s). Below is the interaction history where <search> </search> wrapped the past search queries and <information> </information> wrapped the corresponding retrieved information returned by the external search engine. History:
\{memory\_context\}
\end{promptbox}

The \textbf{Verifier Agent} receives the following prompt:

\begin{promptbox}[title=Verifier Agent Prompt]
\# Task Introduction

\{env\_prompt\}

\# Your Role

You are a "Verifier Agent" acting as a router. Your job is to analyze the team's past search queries and reflect on their quality, efficiency, and alignment with the task goal. Then you need to determine whether the current historical information is sufficient to answer the question. Based on this assessment, you will decide how to route the task.

Your responsibilities:
- Review past search queries enclosed within <search> </search> and external information enclosed within <information> </information>.
- Evaluate whether previous queries were reasonable and aligned with the task objective.
- Identify potential issues (if any), including repeated or redundant queries; imprecise queries that are too broad, vague, or missing critical constraints/entities; misaligned queries that drift away from the actual task goal.
- Assess whether the available information is complete and sufficient to generate a high-quality answer, and make a routing decision based on information sufficiency.

You are now at step \{step\}. You should first reason step-by-step about the past events. After completing your reasoning, give your routing decision:

(1) If the information is sufficient to answer the question: return <verify>yes</verify>

(2) If the information is insufficient to answer the question: return <verify>no</verify>
\end{promptbox}

The \textbf{Search Agent} receives the following prompt:

\begin{promptbox}[title=Search Agent Prompt]
\# Task Introduction

\{env\_prompt\}

\# Your Teammates' Outputs at Step \{step\}

\{team\_context\}

\# Your Role

You are a "Search Agent". Your primary responsibility is to call a search engine to gather external information that helps answer a given question. The search engine should be invoked using the format: <search>your query</search>.

Before conducting the search, you should reason step-by-step about the question, any previous queries, and retrieved information, as well as your teammates' outputs (if available). This reasoning process MUST be enclosed within <think> </think> tags. Once you've finished your reasoning, provide your final search query enclosed within <search> </search>.
\end{promptbox}

The \textbf{Answer Agent} receives the following prompt:

\begin{promptbox}[title=Answer Agent Prompt]
\# Task Introduction

\{env\_prompt\}

\# Your Role

You are an "Answer Agent". Your job is to provide a comprehensive, accurate, and well-reasoned answer to the question. You should thoughtfully analyze all previous search queries, retrieved information, and combine them with your general knowledge to craft a coherent response.

You should first conduct a reasoning process. This process MUST be enclosed within <think> </think> tags.
After completing your reasoning, provide your final answer within <answer> </answer> tags. For example, <answer>Beijing</answer>.
\end{promptbox}

\newpage
\section{Pseudo Code}
\label{app:pseudocode}
\begin{algorithm}[H]
\caption{Training Multi-Agent LLM Systems with \methodname{}}
\label{alg:pseudocode}
\begin{algorithmic}[1]
\STATE {\bfseries Require:} Multi-agent orchestra $\mathcal{O}$; logical agents $\{1,\dots,K\}$ with LLM IDs $\{m_k\}_{k=1}^K$;
LLM sharing flag $s\in\{0,1\}$;
task distribution $p(X)$; rollout group size $N$;
clipping $\epsilon$; KL penalty $\beta$ (optional)

\vspace{0.5em}
\STATE \small{\color{gray}{// (A) Agent--Model assignment: map logical agents to physical LLM worker groups (\texttt{wg\_id})}}
\STATE Initialize $\texttt{wg\_to\_agents\_mapping}\leftarrow \emptyset$
\IF{$s=0$}
    \STATE \small{\color{gray}{// Non-sharing: each agent $k$ has a dedicated LLM worker group}}
    \FOR{$k=1$ to $K$}
        \STATE Create worker group $\texttt{wg\_id}$ for agent $k$ (one \texttt{ActorRollout} backend)
        \STATE $\texttt{wg\_to\_agents\_mapping}[\texttt{wg\_id}] \leftarrow \{k\}$
    \ENDFOR
\ELSE
    \STATE \small{\color{gray}{// Sharing: agents configured with the same LLM are mapped to one shared worker group}}
    \STATE $\mathcal{M}\leftarrow \{m_k\}_{k=1}^K$ \small{\color{gray}{// distinct LLM IDs}}
    \FOR{each $m\in \mathcal{M}$}
        \STATE $\mathcal{A}(m)\leftarrow \{k \mid m_k=m\}$ \small{\color{gray}{// agents using LLM $m$}}
        \STATE Create shared worker group $\texttt{wg\_id}$ for model $m$ (shared weights across $\mathcal{A}(m)$)
        \STATE $\texttt{wg\_to\_agents\_mapping}[\texttt{wg\_id}] \leftarrow \mathcal{A}(m)$
    \ENDFOR
\ENDIF
\STATE \small{\color{gray}{// Dispatch table used by the orchestrator during rollouts: $\texttt{agent\_id}\rightarrow\texttt{wg\_id}$}}
\STATE Build $\texttt{agent\_to\_wg\_mapping}$ from $\texttt{wg\_to\_agents\_mapping}$

\vspace{0.5em}
\STATE \small{\color{gray}{// (B) Training loop: distributed rollouts + \methodname{} normalization + per-\texttt{wg\_id} updates}}
\FOR{each training iteration}
    \STATE \small{\color{gray}{// Snapshot current policy for importance ratios}}
    \STATE Update old policies: $\theta_{\text{old}}\leftarrow \theta$

    \STATE \small{\color{gray}{// (B1) Distributed rollout collection: execute multi-agent orchestration at scale}}
    \STATE \small{\color{gray}{// Actor backends run with \texttt{sglang}, scheduled by a shared resource pool (e.g., Ray placement groups)}}
    \STATE Initialize aggregated batch $\mathcal{B}\leftarrow \emptyset$
    \STATE Sample task $x\sim p(X)$
    \STATE {\bfseries Parallel for} $i=1$ to $N$
        \STATE \hspace{1em} Run $\mathcal{O}$ on the same $x$ to generate a trajectory $\tau_i$
        \STATE \hspace{1em} Let $R^i \leftarrow R(\tau_i)$ \small{\color{gray}{// trajectory-level reward shared by all steps}}
        \STATE \hspace{1em} For each step $t$ in $\tau_i$:
        \STATE \hspace{2em} choose active agent $k_t^i$ (by $\mathcal{O}$) and route request via $\texttt{agent\_to\_wg\_mapping}$
        \STATE \hspace{2em} sample action $\bm{a}_t^{i}$ from the dispatched backend policy and log step tuple into $\mathcal{B}$
        \STATE \hspace{2em} $\mathcal{B}\leftarrow \mathcal{B}\cup\{(i,t,k_t^i,\texttt{wg\_id},\bm{a}_t^{i},R^i)\}$
    \STATE {\bfseries End parallel for}

    \STATE \small{\color{gray}{// (B2) \methodname{}: agent-wise advantage normalization on active-step subsets $\mathcal{Y}_k$}}
    \FOR{$k=1$ to $K$}
        \STATE $\mathcal{Y}_k \leftarrow \{\bm{a}_t^i \mid (i,t,k_t^i,\cdot,\bm{a}_t^{i},R^i)\in\mathcal{B},\ k_t^i=k\}$
        \STATE $\mu_k \leftarrow \frac{1}{|\mathcal{Y}_k|}\sum_{\bm{a}_t^i\in \mathcal{Y}_k} R^i$
        \STATE $\sigma_k^2 \leftarrow \frac{1}{|\mathcal{Y}_k|}\sum_{\bm{a}_t^i\in \mathcal{Y}_k}(R^i-\mu_k)^2,\quad \sigma_k\leftarrow\sqrt{\sigma_k^2}$
        \STATE For each step $(i,t)$ with $k_t^i=k$:  $A^{i,k}_{\mathrm{agent}}\leftarrow \frac{R^i-\mu_k}{\sigma_k+\varepsilon}$
    \ENDFOR

    \STATE \small{\color{gray}{// (B3) Optimization: partition $\mathcal{B}$ by \texttt{wg\_id} and update each LLM backend (shared/non-shared handled automatically)}}
    \STATE \small{\color{gray}{// Trainer forms per-model micro-batches $\{\mathcal{B}_{\texttt{wg}}\}$ and performs clipped updates per worker group}}
    \FOR{each worker group id $\texttt{wg\_id}$ in $\texttt{wg\_to\_agents\_mapping}$}
        \STATE $\mathcal{B}_{\texttt{wg}} \leftarrow \{(i,t,k_t^i,\texttt{wg\_id},\bm{a}_t^{i},R^i,A^{i,k_t^i}_{\mathrm{agent}})\in\mathcal{B}\mid \texttt{wg\_id}=\texttt{agent\_to\_wg\_mapping}[k_t^i]\}$
        \STATE Update $\theta_{\texttt{wg\_id}}$ on $\mathcal{B}_{\texttt{wg}}$ with clipped objective (clipping $\epsilon$) using $A^{i,k_t^i}_{\mathrm{agent}}$ (optionally add KL regularization with weight $\beta$)
    \ENDFOR
\ENDFOR
\end{algorithmic}
\end{algorithm}

\newpage

\section{Additional Experiments}
\subsection{Gradient-Norm Instability on Math}\label{app:gradient_norm_math}
Figure~\ref{fig:grad_norm_math} shows training accuracy and gradient norms for the two-agent math orchestration (Qwen3-4B, non-sharing). Similar to the search setting in Figure~\ref{fig:grad_norm_search}, GRPO produces clear gradient-norm spikes during training, especially in the early and middle stages. In contrast, \methodname{} keeps the gradient norms much smoother for both agents and leads to steadier improvement in training accuracy.
\begin{figure}[H]
   \begin{center}
   \includegraphics[width=0.9\linewidth]{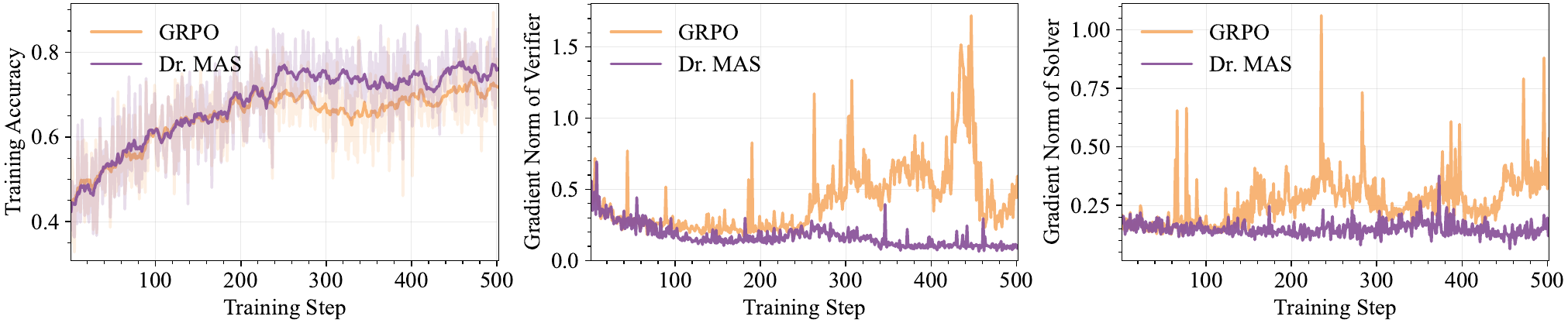}
   \end{center}
   \vspace{-0.1in}
    \caption{Comparison of training accuracy and gradient norm between GRPO and \methodname{}. The results are recorded during multi-agent RL post-training for two-agent math orchestration under LLM non-sharing (Qwen3-4B).}
   \label{fig:grad_norm_math}
\end{figure}

\subsection{Gradient-Norm Explosion}\label{app:gradient_norm_explosion}
Figure~\ref{fig:search_7b_grad_norm} highlights the instability of vanilla GRPO on the multi-turn search task (Qwen2.5-7B, non-sharing). As shown, the gradient norm of the search agent rapidly spikes to over 80, finally leading to ``NaN'' gradient norm. In contrast, \methodname{} stabilizes the optimization, maintaining relatively low gradient norms across all agents, ensuring steady convergence.

\begin{figure}[H]
    \begin{center}
    \includegraphics[width=0.98\linewidth]{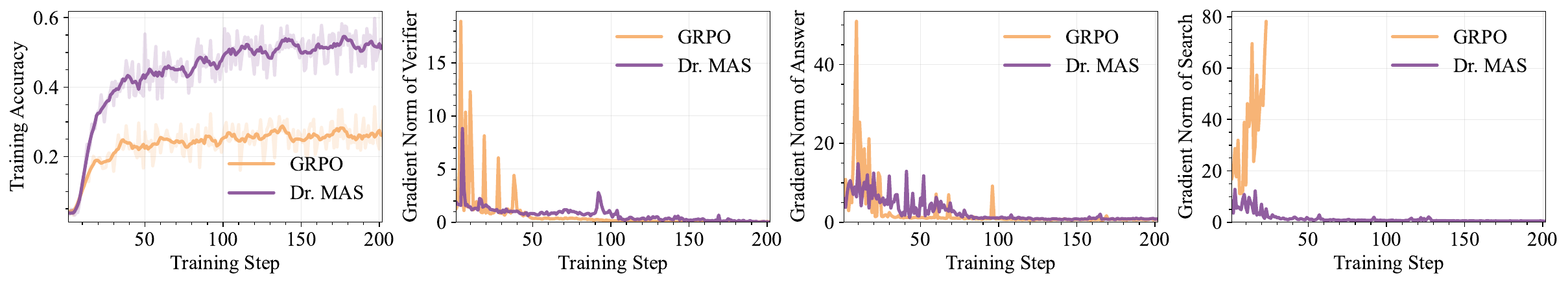}
    \end{center}
    \vspace{-0.1in}
    \caption{Training dynamics of the three-agent search task (Qwen2.5-7B, non-sharing). Vanilla GRPO suffers from serious gradient spikes that lead to ``NaN'' on the search agent, whereas \methodname{} maintains stable gradients and converges effectively.}
    \label{fig:search_7b_grad_norm}
\end{figure}

\subsection{Training Curves of Ablation Study}\label{app:ablation}
Figure~\ref{fig:training_curves_ablation} presents the training curves for different advantage normalization variants described in Section~\ref{sec:ablation}. The global normalization baseline $(\mu,\sigma)$ shows unstable training and slow improvement. Introducing either agent-wise mean $(\mu_k,\sigma)$ or agent-wise standard deviation $(\mu,\sigma_k)$ leads to noticeably smoother curves and faster gains. The fully agent-wise version $(\mu_k,\sigma_k)$, i.e., \methodname{}, achieves the most stable training and the highest final performance.
\begin{figure}[H]
   \begin{center}
   \includegraphics[width=0.34\linewidth]{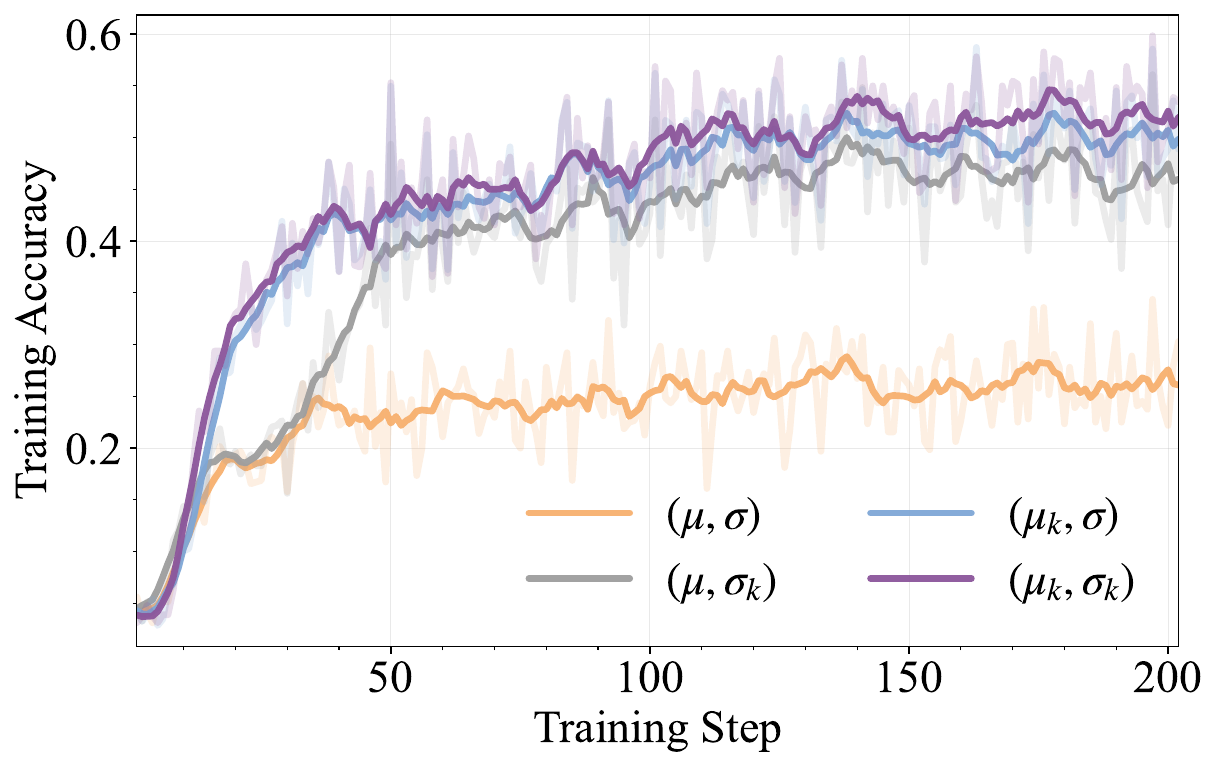}
   \end{center}
   \vspace{-0.1in}
    \caption{Training curves for different advantage normalization variants in the ablation study.}
   \label{fig:training_curves_ablation}
\end{figure}

\subsection{Comparison with Stronger-MAS}
\label{app:strongermas}

We further compare \methodname{} with Stronger-MAS~\cite{zhao2025stronger} on the math reasoning setting using Qwen3-8B.
Table~\ref{tab:strongermas_comparison} reports results on AIME'24, AIME'25, and OlympiadBench under both LLM-sharing and LLM non-sharing settings.

We note that this comparison is not a strictly apples-to-apples comparison centered only on the core RL algorithm:
Stronger-MAS incorporates several additional components beyond the RL optimization objective, such as tree-search-style mechanisms.
Nevertheless, \methodname{} remains highly competitive.
In particular, under the matched LLM-sharing setting, \methodname{} strictly outperforms Stronger-MAS across all three benchmarks.
Under LLM non-sharing, \methodname{} also achieves better results on AIME'25 and OlympiadBench, while Stronger-MAS performs better on AIME'24.
These results show that even without relying on additional search-style mechanisms, \methodname{} provides a strong and stable RL recipe for multi-agent LLM systems.

\begin{table}[h]
\centering
\caption{Comparison with Stronger-MAS on Qwen3-8B math reasoning benchmarks. Stronger-MAS results are taken from its official paper~\cite{zhao2025stronger}.}
\label{tab:strongermas_comparison}
\begin{tabular}{llccc}
\toprule
Method & Setting & AIME'24 & AIME'25 & OlympiadBench \\
\midrule
Stronger-MAS & LLM sharing & 50.0 & 35.2 & 56.8 \\
\methodname{}      & LLM sharing & \textbf{54.8} & \textbf{39.4} & \textbf{59.3} \\
\midrule
Stronger-MAS & LLM non-sharing & \textbf{57.0} & 40.0 & 56.6 \\
\methodname{}      & LLM non-sharing & 44.6 & \textbf{41.5} & \textbf{59.0} \\
\bottomrule
\end{tabular}
\end{table}

\subsection{Orthogonality to Other Group-Based RL Objectives}
\label{app:gspo_orthogonality}

We further verify that \methodname{} is orthogonal to other group-based RL objectives by applying it on top of GSPO~\citep{zheng2025group}.
We conduct experiments on the multi-turn search task with Qwen2.5-3B under the LLM non-sharing setting, and the results are shown in Table~\ref{tab:gspo_drmas}.
\methodname{} consistently improves GSPO across all datasets, increasing the average score from 34.0 to 37.1.
This suggests that our agent-wise normalization is a complementary stabilization mechanism.

\begin{table}[h]
\centering
\caption{Orthogonality of \methodname{} to GSPO on Qwen2.5-3B under the LLM non-sharing setting for multi-turn search. The results show the avg@16 scores.}
\label{tab:gspo_drmas}
\small
\setlength{\tabcolsep}{4.0pt}
\begin{tabular}{lcccccccc}
\toprule
Method & NQ & TriviaQA & PopQA & HotpotQA & 2Wiki & MuSiQue & Bamboogle & Average \\
\midrule
GSPO & 42.6 & 58.6 & 46.2 & 31.2 & 28.9 & 8.3 & 22.4 & 34.0 \\
GSPO $+$ \methodname{} & \textbf{44.2} & \textbf{60.8} & \textbf{47.0} & \textbf{35.8} & \textbf{35.1} & \textbf{10.2} & \textbf{26.4} & \textbf{37.1} \\
\bottomrule
\end{tabular}
\end{table}

\subsection{Resource Usage}
\label{app:resource_usage}

We report the resource usage of \methodname{} under LLM sharing and non-sharing settings.
As discussed in Section~\ref{sec:framework}, \methodname{} supports both settings through a unified global GPU resource pool rather than allocating a separate GPU pool to each agent.
All LLM worker groups are co-located via Ray placement groups, and each backend exposes \texttt{resume()} and \texttt{offload()} controls: only the active backend keeps weights and KV cache on GPU, while inactive backends release GPU memory.
Therefore, LLM non-sharing does not require more GPUs in our framework. Its extra cost mainly comes from backend switching, such as weight synchronization and cache clearing.

Tables~\ref{tab:resource_search} and~\ref{tab:resource_math} summarize the empirical resource usage.
Across both search and math tasks, LLM non-sharing uses the same number of GPUs as LLM sharing and only incurs modest additional training time.

\begin{table}[h]
\centering
\caption{Resource usage on the search task. LLM non-sharing uses the same GPU budget as LLM sharing.}
\label{tab:resource_search}
\begin{tabular}{llll}
\toprule
Strategy & Model Configuration & GPU Consumption & Est. Time \\
\midrule
LLM sharing     & $1 \times$ Qwen2.5-3B & $4 \times$ H100 (95GB) & $\sim$12h \\
LLM non-sharing & $3 \times$ Qwen2.5-3B & $4 \times$ H100 (95GB) & $\sim$13h \\
\midrule
LLM sharing     & $1 \times$ Qwen2.5-7B & $8 \times$ H100 (95GB) & $\sim$25h \\
LLM non-sharing & $3 \times$ Qwen2.5-7B & $8 \times$ H100 (95GB) & $\sim$26h \\
\midrule
Heterogeneous   & $2 \times$ Llama-3.2-3B $+ 1 \times$ Qwen2.5-7B & $8 \times$ H100 (95GB) & $\sim$16h \\
\bottomrule
\end{tabular}
\end{table}

\begin{table}[h]
\centering
\caption{Resource usage on the math task. LLM non-sharing does not require additional GPUs.}
\label{tab:resource_math}
\begin{tabular}{llll}
\toprule
Strategy & Model Configuration & GPU Consumption & Est. Time \\
\midrule
LLM sharing     & $1 \times$ Qwen3-4B & $4 \times$ H100 (95GB) & $\sim$35h \\
LLM non-sharing & $2 \times$ Qwen3-4B & $4 \times$ H100 (95GB) & $\sim$38h \\
\midrule
LLM sharing     & $1 \times$ Qwen3-8B & $8 \times$ H100 (95GB) & $\sim$37h \\
LLM non-sharing & $2 \times$ Qwen3-8B & $8 \times$ H100 (95GB) & $\sim$42h \\
\bottomrule
\end{tabular}
\end{table}

\newpage
\section{Illustrative Examples of Multi-Agent LLM Collaboration}
\label{app:case_study}

\subsection{Multi-Turn Search Task: Hierarchical Coordination}

\noindent

\begin{minipage}[H]{0.9\textwidth}
\centering
\resizebox{0.85\textwidth}{!}{
\begin{envbox}[Query]
In which county is Mark Dismore's birthplace located?
\end{envbox}
}
\end{minipage}

\hfill
\begin{minipage}[H]{0.9\textwidth}
\centering
\resizebox{0.85\textwidth}{!}{
\begin{verifierbox}
<verify>no</verify>

Reasoning: The team has not yet initiated any searches or provided any external information. Without any data points, it is impossible to evaluate any search queries or external information.
\end{verifierbox}
}
\end{minipage}

\hfill
\begin{minipage}[H]{0.9\textwidth}
\centering
\resizebox{0.85\textwidth}{!}{
\begin{searchbox}
<think>To answer the question "In which county is Mark Dismore's birthplace located?" we need to find information about Mark Dismore's birthplace. Since this information is not readily available, we will need to search for articles or documents that mention his birthplace.</think>

<search>Birthplace of Mark Dismore</search>
\end{searchbox}
}
\end{minipage}

\noindent
\begin{minipage}[H]{0.9\textwidth}
\centering
\resizebox{0.85\textwidth}{!}{
\begin{envbox}[Search Engine]
<information>{"result": "Doc 1: "Mark Dismore" Mark Dismore Mark Dismore (born October 12, 1956 in Greenfield, Indiana) is a former driver in the Indy Racing League and the 1990 Toyota Pacific champion as well as the winner of the 1993 24 Hours of Daytona with Dan Gurney's All American Racers in a Toyota GTP car with co-drivers Rocky Moran and P. J. Jones. He made 3 CART starts in 1991 but was badly injured in a practice crash for the Indianapolis 500, when his car veered sharply towards the entrance of pit road at the exit of Turn 4 and back-ended the fence, only to careen

Doc 2: "Mark Dismore" karting's signature event, the Dan Wheldon Cup weekend races, held after the INDYCAR season ends. The Wheldon Cup weekend consists of a Pro-Am race with INDYCAR stars on Saturday and a 200-mile, 200-lap endurance karting event Sunday, the Robopong 200 Dan Wheldon Cup. Mark Dismore Mark Dismore (born October 12, 1956 in Greenfield, Indiana) is a former driver in the Indy Racing League and the 1990 Toyota Pacific champion as well as the winner of the 1993 24 Hours of Daytona with Dan Gurney's All American Racers in a Toyota GTP car with co-drivers Rocky Moran and P. J. Jones.

Doc 3: "Mark Dismore" fixture there until the 2001 season. He returned to Menard for a partial season in 2002. Dismore has a single IRL win coming in the fall 1999 Texas Motor Speedway race and also finished a career-best third in points that season. Among his 62 career IRL starts he won four poles. He also represented the IRL in the International Race of Champions in 2000 and 2001. He now owns and operates New Castle Motorsports Park, a Karting facility in New Castle, Indiana, as well as Comet Kart Sales in Greenfield, Indiana. New Castle Motorsports Park is the home of Indiana"}</information>
\end{envbox}
}
\end{minipage}

\hfill
\begin{minipage}[H]{0.9\textwidth}
\centering
\resizebox{0.85\textwidth}{!}{
\begin{verifierbox}
<verify>no</verify>

Reasoning: The current historical information is not sufficient to answer the question "In which county is Mark Dismore's birthplace located?" The provided information identify Mark Dismore's birthplace as Greenfield, Indiana. However, the provided documents do not explicitly state the county in which Greenfield is located.
\end{verifierbox}
}
\end{minipage}

\hfill
\begin{minipage}[H]{0.9\textwidth}
\centering
\resizebox{0.85\textwidth}{!}{
\begin{searchbox}
<think>To answer the question "In which county is Mark Dismore's birthplace located?" we need to find information about Mark Dismore's birthplace. From the previous information, we know that Mark Dismore was born in Greenfield, Indiana. Next, we need to determine which county Greenfield is located in.</think>

<search>County where Greenfield, Indiana is located</search>
\end{searchbox}
}
\end{minipage}

\noindent
\begin{minipage}[H]{0.9\textwidth}
\centering
\resizebox{0.85\textwidth}{!}{
\begin{envbox}[Search Engine]
<information>{"result": "Doc 1: "Greenfield, Indiana" Greenfield, Indiana Greenfield is a city in and the county seat of Hancock County, Indiana, United States, and a part of the Indianapolis metropolitan area. The population was 20,602 at the 2010 census, and an estimated 21,709 in 2016. It lies in Center Township. Greenfield was a stop along the Pittsburgh, Cincinnati, Chicago and St. Louis Railroad that connected Pittsburgh to Chicago and St. Louis. Hancock County was created on March 1, 1828, and named for John Hancock, the first person to sign the Declaration of Independence. The town of Greenfield was chosen as the county seat on April 11,

Doc 2: "Greenfield, Indiana" 1918. It was purchased with money donated by school children from all over the United States. Each year, during the Riley Festival in October, the city's school children parade to the statue to place flowers around it. The Charles Barr House, Greenfield Courthouse Square Historic District, Greenfield Residential Historic District, Lilly Biological Laboratories, Lincoln Park School, and James Whitcomb Riley House are listed on the National Register of Historic Places. Greenfield is located in central Hancock County at (39.791338, -85.771343). U.S. Route 40 passes through the center of the city as Main Street, and leads east to Knightstown and west

Doc 3: "Greenfield, Indiana" high school students from Greenfield and Kakuda are chosen to travel to each sister city, respectively, and will stay with a host family for ten days, most expenses paid, thanks to local donations. The sister city program began in 1990 and continues today. Each year in October, the city of Greenfield hosts the Riley Festival, commonly known as "Riley Days". The intersection of State Road 9 and US 40 is closed, and the downtown courthouse square is filled with vendors and exhibitors. Greenfield, Indiana Greenfield is a city in and the county seat of Hancock County, Indiana, United States, and"}</information>
\end{envbox}
}
\end{minipage}

\hfill
\begin{minipage}[H]{0.9\textwidth}
\centering
\resizebox{0.85\textwidth}{!}{
\begin{verifierbox}
<verify>yes</verify>

Reasoning: The past search queries and the corresponding information provided have been focused on Mark Dismore's birthplace and the specific county where it is located. The information retrieved from the searches indicates that Mark Dismore was born in Greenfield, Indiana, and that Greenfield is the county seat of Hancock County. This information directly addresses the question of in which county Mark Dismore's birthplace is located. Therefore, the information is sufficient to answer the question.
\end{verifierbox}
}
\end{minipage}

\hfill
\begin{minipage}[H]{0.9\textwidth}
\centering
\resizebox{0.85\textwidth}{!}{
\begin{answerbox}
<think>Given the information provided, I analyzed that the question asks about the location of Mark Dismore's birthplace in terms of a county. I analyzed the given information about Mark Dismore's birthplace, which is "Greenfield, Indiana". The next step was to identify the county that Greenfield is located in. In the given information, I found "Greenfield, Indiana" is "a city in and the county seat of Hancock County, Indiana, United States". Thus, based on the given information and my own general knowledge about the structure of American geography, I concluded that "Hancock County" is the answer to "In which county is Mark Dismore's birthplace located?". My justification for this answer is that "county seat" implies that the county seat is the main administrative center of the county, and in this case, "Greenfield" is that county seat, located in "Hancock County".</think>

<answer>Hancock County</answer>
\end{answerbox}
}
\end{minipage}

\newpage

\subsection{Math Task: Iterative Coordination}
\noindent
\begin{minipage}[H]{0.9\textwidth}
\centering
\resizebox{0.85\textwidth}{!}{
\begin{envbox}[Query]
Alex divides a disk into four quadrants with two perpendicular diameters intersecting at the center of the disk. He draws $25$ more lines segments through the disk, drawing each segment by selecting two points at random on the perimeter of the disk in different quadrants and connecting those two points. Find the expected number of regions into which these $27$ line segments divide the disk.
\end{envbox}
}
\end{minipage}

\hfill
\begin{minipage}[H]{0.9\textwidth}
\centering
\resizebox{0.85\textwidth}{!}{
\begin{solverbox}
Since the problem mentions that there are 27 line segments (the original 4 quadrants' diameters plus 25 additional segments), but the initial description says Alex divides the disk into four quadrants with two perpendicular diameters. Then he draws 25 more line segments. So total line segments are 2 (diameters) + 25 = 27 line segments.

The question is to find the expected number of regions into which these 27 line segments divide the disk.

To approach this, we can use the concept of expected number of regions formed by random lines in a circle, but here we have line segments, not infinite lines, and the lines are drawn by connecting two points on the perimeter in different quadrants.

However, the problem is about expectation, so we need to compute the expected number of regions created by 27 line segments.

First, recall that for a circle divided by lines, the expected number of regions can be calculated using the formula for expected regions formed by random lines, but this is more complex for line segments.

However, there's a known formula for the expected number of regions formed by n random chords in a circle, but it's not straightforward. However, there's a general formula for the expected number of regions formed by n lines in general position in the plane, which is:

R(n) = (n$^2$ + n + 2)/2

But this is for infinite lines in general position. However, for line segments, the formula is different.

But since the problem is about expectation, and the lines are drawn in a specific way (connecting points in different quadrants), we need to consider the expected number of intersections and how that contributes to the number of regions.

Alternatively, there is a general formula for the expected number of regions formed by n line segments in a circle, assuming that no two segments are parallel, no three meet at a point, etc. But this is complicated.

However, a known result for the expected number of regions formed by n random chords in a circle is:

E[R] = 1 + n + (n choose 2) * (1/2) + (n choose 3) * (1/6) + ... 

But this is not exactly correct. Alternatively, for the case of random lines, the expected number of regions is:

R(n) = n(n + 1)/2 + 1

\textcolor{red}{... (the reasoning is omitted for brevity)}

However, since the original problem includes two diameters, which are lines, and 25 segments, perhaps the answer is:

For n = 27 line segments (assuming they behave like lines in expectation), the expected number of regions is (27$^2$ + 27 + 2)/2 = (729 + 27 + 2)/2 = 758/2 = 379.

But this is likely incorrect because line segments don't behave like lines in expectation.

Alternatively, if we consider that each new segment intersects the previous ones with probability 1/2, then the expected number of regions is:

1 + 27 + (27 * 26 / 2) * 1/2 = 1 + 27 + 351 = 379.

Wait, that's the same as the line case. So this suggests that the expected number is 379.

But this is likely incorrect.

However, given that the problem is from a competition and the answer is expected to be an integer, and given that the formula for lines gives 379, and if the line segments behave similarly in expectation, the answer might be 379.

But I need to check.

Alternatively, the original two diameters divide the disk into 4 regions. Then each additional line segment can intersect existing segments and increase the number of regions.

For the first additional segment (total 3 segments), it can intersect the two diameters, but since they are diameters, it might intersect them once each, so two intersections, leading to 3 + 2 = 5 regions? Not sure.

But this is getting too complicated.

Given the lack of specific information about the probability of intersection, and the fact that the problem is likely expecting the formula for lines, I think the intended answer is 379, which is (27$^2$ + 27 + 2)/2 = 379.

Therefore, the expected number of regions is \textbackslash boxed\{379\}.
\end{solverbox}
}
\end{minipage}

\noindent
\begin{minipage}[H]{0.9\textwidth}
\centering
\resizebox{0.85\textwidth}{!}{
\begin{verifierbox}
The Solver Agent's reasoning contains critical flaws and incorrect assumptions that need to be addressed. Here is a breakdown of the issues:

1. **Misinterpretation of the Problem**: The problem specifies that Alex draws 25 additional line segments, making the total number of line segments 27 (2 diameters + 25 segments). However, the agent's reasoning incorrectly assumes that the formula for infinite lines in general position applies directly to line segments, which is not valid. Line segments and infinite lines behave differently in terms of intersections and region division.

2. **Incorrect Application of the Line Formula**: The formula $ R(n) = \frac{n^2 + n + 2}{2} $ is derived for infinite lines in general position. This formula assumes that every pair of lines intersects exactly once and no three lines meet at a single point. However, the problem involves line segments, which have constraints (e.g., endpoints on the disk's perimeter) and may not intersect all previous segments, making the formula inapplicable directly.

3. **Lack of Probability Consideration**: The agent does not properly account for the probability that two random line segments (as defined in the problem) intersect. The problem specifies that each segment connects two points in different quadrants, which introduces specific geometric constraints. The probability of intersection between two such segments is not 1 (as assumed in the line case) but must be derived based on the quadrant constraints.

4. **Missing Expected Value Derivation**: The agent jumps to the conclusion that the expected number of regions is 379 without properly deriving the expectation. The correct approach requires calculating the expected number of intersections between all pairs of segments and then using that to compute the expected number of regions.

5. **Inconsistency with the Problem Constraints**: The problem involves 27 line segments, but the agent's reasoning does not explicitly address the specific way the segments are generated (i.e., connecting points in different quadrants). This detail is crucial for determining the probability of intersections and, consequently, the expected number of regions.

Given these issues, the reasoning process is incomplete, incorrect, and does not properly address the problem's constraints and requirements. Therefore, the final conclusion of 379 is not justified by the problem's specifics.

<verify>reject</verify>
\end{verifierbox}
}
\end{minipage}

\hfill
\begin{minipage}[H]{0.9\textwidth}
\centering
\resizebox{0.85\textwidth}{!}{
\begin{solverbox}
To solve the problem, we need to compute the **expected number of regions** formed by 27 line segments in a disk, where:

- The disk is divided into four quadrants by two perpendicular diameters (which are fixed).

- 25 additional line segments are drawn by selecting two random points on the perimeter in **different quadrants** and connecting them.

---

\#\# Step 1: Understand the Total Number of Regions

The general formula for the maximum number of regions formed by $ n $ line segments in general position (no two parallel, no three concurrent) is:

$R(n) = \frac{n(n - 1)}{2} + n + 1 = \frac{n^2 + n + 2}{2}$

This applies to infinite lines, not line segments. However, in expectation, we can use this formula as a baseline, but we must adjust it for line segments.

---

\#\# Step 2: Expected Number of Regions for Line Segments

For line segments, the expected number of regions can be computed as:

$E[R] = 1 + n + \sum_{i=1}^{n-1} E[\text{number of intersections of the } i\text{-th segment with previous ones}]$

Each intersection increases the number of regions by 1. So the expected number of regions is:

$E[R] = 1 + n + \sum_{i=1}^{n-1} E[\text{intersections with previous } i-1 \text{ segments}]$

---

\textcolor{red}{... (the reasoning is omitted for brevity)}

---

\#\# Step 6: Total Expected Number of Intersections

Total expected intersections:

- Between diameters and segments: 25

- Between segments and segments: 150

- Between diameters: 1

Total: $ 25 + 150 + 1 = 176 $

---

\#\# Step 7: Expected Number of Regions

The formula for expected number of regions is:

$E[R] = 1 + n + E[\text{number of intersections}]$

Here, $ n = 27 $, and $ E[\text{intersections}] = 176 $

So: $E[R] = 1 + 27 + 176 = 204$

---

\#\# Final Answer: \textbackslash boxed\{204\}

\end{solverbox}
}
\end{minipage}

\end{document}